\documentclass{article}
\usepackage{log_2025}
\usepackage{booktabs}			
\usepackage{multirow}			
\usepackage{amsfonts}		
\usepackage{graphicx}						
\usepackage{duckuments}
\usepackage[numbers,compress,sort]{natbib}	
\usepackage{amsthm}
\usepackage{tikz}
\usepackage{subcaption}
\usepackage{thmtools} 
\usepackage{enumitem}
\usetikzlibrary{arrows.meta, positioning}
\usepackage[table]{xcolor}
\usepackage[dvipsnames]{xcolor}
\usepackage{url} 

\newcommand{\mb}[1]{\mathbf{#1}}
\newcommand{\R}{\mathbb{R}}
\newcommand{\C}{\mathbb{C}}

\newtheorem{definition}{Definition}
\newtheorem{proposition}{Proposition}
\newtheorem{remark}{Remark}

\newtheorem{notation}{Notation}

\newcommand{\new}[1]{{#1}}

\title[Complex-Weighted Convolutional Networks: Provable Expressiveness via Complex Diffusion]{Complex-Weighted Convolutional Networks: Provable Expressiveness via Complex Diffusion}

\author[Cristina L\'opez Amado]{%
Cristina L\'opez Amado\\
IST Austria\\
\email{cristina.lopezamado@ist.ac.at}\And
Tassilo Schwarz\\
Mathematical Institute\\University of Oxford\thanks{and Mathematical bioPhysics Group, Max Planck Institute for Multidisciplinary Sciences}\\
\email{tassilo.schwarz@maths.ox.ac.uk}\And
Yu Tian\\
Center for Systems Biology Dresden\thanks{Max Planck Institute of Molecular Cell Biology and Genetics, Max Planck Institute for the Physics of Complex Systems, and Dresden University of Technology}\\
\email{yu.tian.research@gmail.com}\And
Renaud Lambiotte\\
Mathematical Institute\\University of Oxford\\
\email{renaud.lambiotte@maths.ox.ac.uk}
}

\begin{document}

\maketitle

\begin{abstract}
Graph Neural Networks (GNNs) have achieved remarkable success across diverse applications, yet they remain limited by oversmoothing and poor performance on heterophilic graphs. To address these challenges, we introduce a novel framework that equips graphs with a complex-weighted structure, assigning each edge a complex number to drive a diffusion process that extends random walks into the complex domain. We prove that this diffusion is highly expressive: with appropriately chosen complex weights, any node-classification task can be solved in the steady state of a complex random walk. Building on this insight, we propose the Complex-Weighted Convolutional Network (CWCN), which learns suitable complex-weighted structures directly from data while enriching diffusion with learnable matrices and nonlinear activations. CWCN is simple to implement, requires no additional hyperparameters beyond those of standard GNNs, and achieves competitive performance on benchmark datasets. Our results demonstrate that complex-weighted diffusion provides a principled and general mechanism for enhancing GNN expressiveness, opening new avenues for models that are both theoretically grounded and practically effective.\footnote{Code: \href{https://github.com/clopezamado/complex-weighted-convolutional-networks.git}{\texttt{https://github.com/clopezamado/complex-weighted-convolutional-networks.git}}} 
\end{abstract}

\section{Introduction}
Graph Neural Networks (GNNs) \cite{scarselli2008graph} have emerged as a powerful class of machine learning models to deal with relational and structured data. They have shown state-of-the-art performance in a wide range of applications, such as recommendation systems \cite{fan2019graph}, molecular property prediction \cite{gilmer2017neural}, \new{webpage classification \cite{webpageClas} or predictions in citation networks \cite{citationNetwork}}.

\paragraph{Limitations of GNNs.} Most Graph Neural Network (GNN) architectures are built on the message passing paradigm, where each node iteratively aggregates information from its neighbors to update its representation. While this approach has been highly successful, “classical” GNNs usually face two well-documented challenges: (1) poor performance on heterophilic graphs \new{\cite{zheng2022graph}} and (2) oversmoothing \new{\cite{rusch2023survey}}. Both stem from the inherent difficulty of capturing long-range dependencies.
The heterophily problem arises because message passing implicitly assumes homophily \new{\cite{zheng2022graph}}, i.e. that neighboring nodes tend to share similar features and labels, an assumption that is often not verified in real-world networks \cite{zhu2020beyond}. Oversmoothing, on the other hand, refers to the tendency of node representations to become indistinguishable after repeated message passing, eroding their discriminative power.
Oversmoothing has been particularly well studied in the context of Graph Convolutional Networks (GCNs). From a diffusion perspective, GCNs can be seen as implementing an augmented heat diffusion process over the graph \new{\cite{li2018deeper}}. In such a process, node features within the same connected component converge to identical values \cite{li2018deeper}, a phenomenon that significantly contributes to oversmoothing in GCNs \cite{oono2020graph,cai2020note}.

\paragraph{Beyond Heat Diffusion.} To address these limitations, recent research has focused on developing more expressive diffusion processes. Prominent examples include Sheaf Convolutional Networks (SCNs) \cite{hansen2020sheaf,bodnar2022neural} and Graph Neural Reaction-Diffusion Networks (GREAD) \cite{choi2023gread}. Both approaches enrich the underlying graph structure to support more sophisticated forms of diffusion.
GREAD has shown strong empirical performance on standard benchmarks, providing a practical and scalable alternative to sheaf-based methods, though it does not come with theoretical guarantees. SCNs, in contrast, offer provable expressiveness, but their guarantees require the sheaf dimension to scale with the number of target classes, a requirement that increases model parameters and, in turn, computational cost.

\begin{figure}
    \centering
\begin{subfigure}{0.96\textwidth}
    \centering
    \begin{tabular}{@{}c@{\hspace{3pt}}c@{\hspace{3pt}} c@{\hspace{3pt}} c@{}}
        \makebox[0.24\textwidth]{\scriptsize{Initial state}} &
        \makebox[0.24\textwidth]{\scriptsize{Diffusion time 1}} &
	    \makebox[0.24\textwidth]{\scriptsize{Diffusion time 5}} &
        \makebox[0.24\textwidth]{\scriptsize{Diffusion time 12}} \\
        \includegraphics[width=0.24\textwidth]{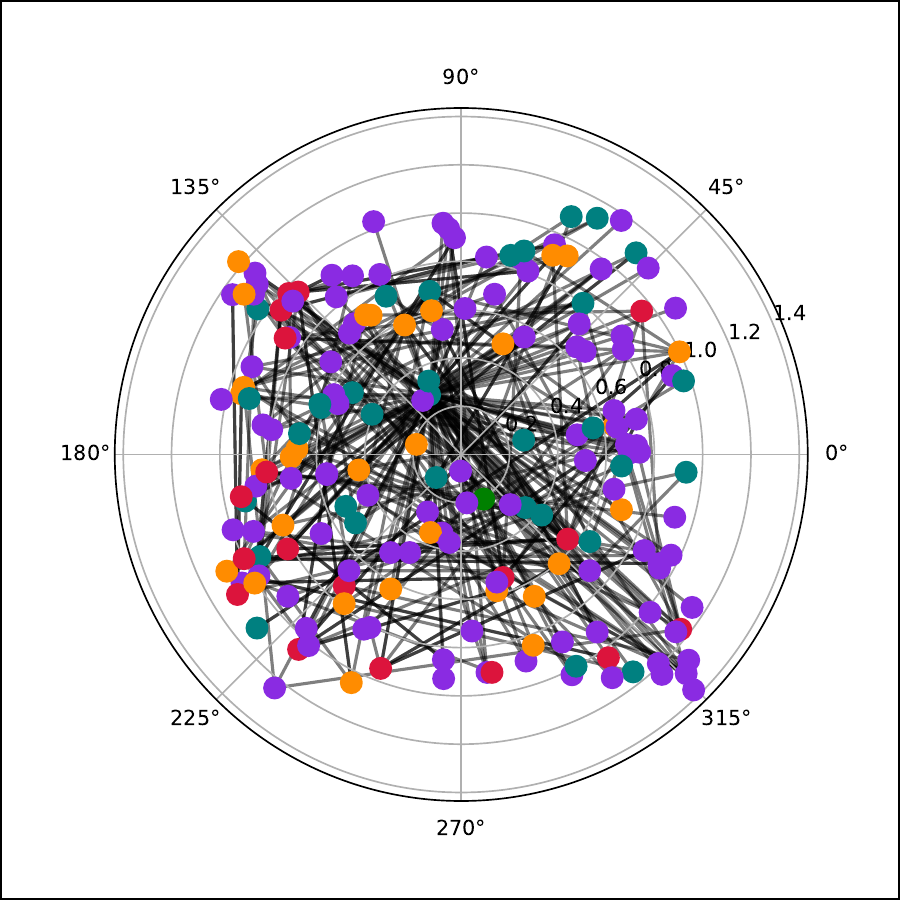} &
        \includegraphics[width=0.24\textwidth]{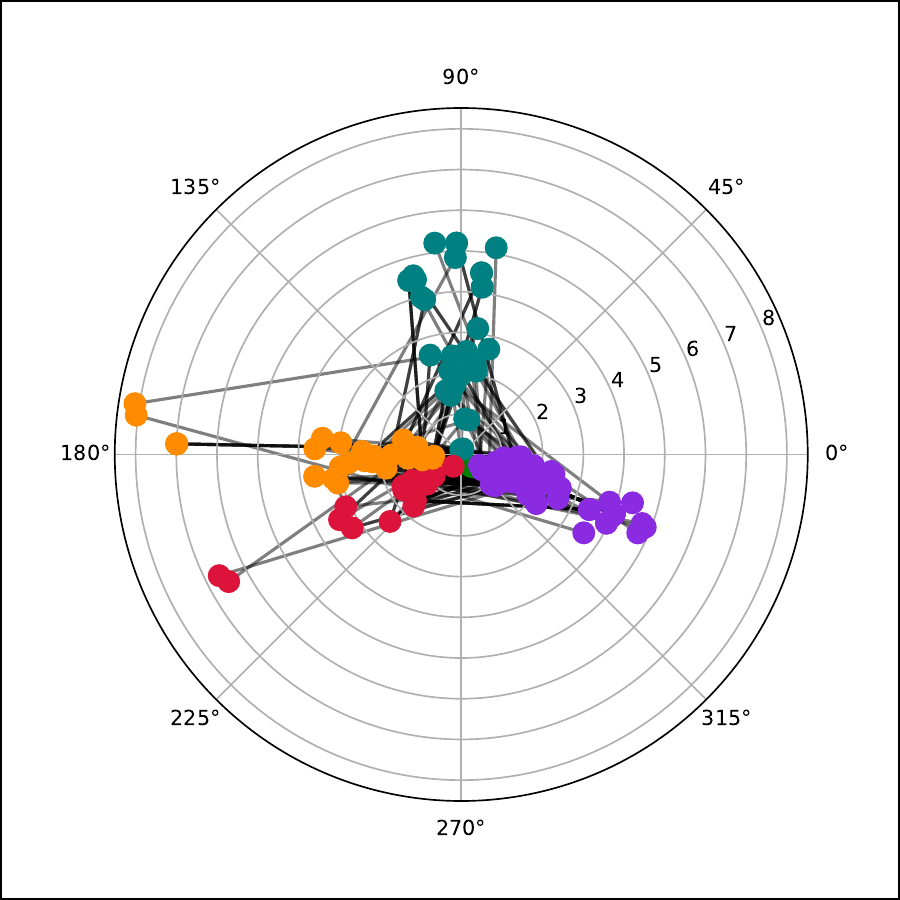} &
        \includegraphics[width=0.24\textwidth]{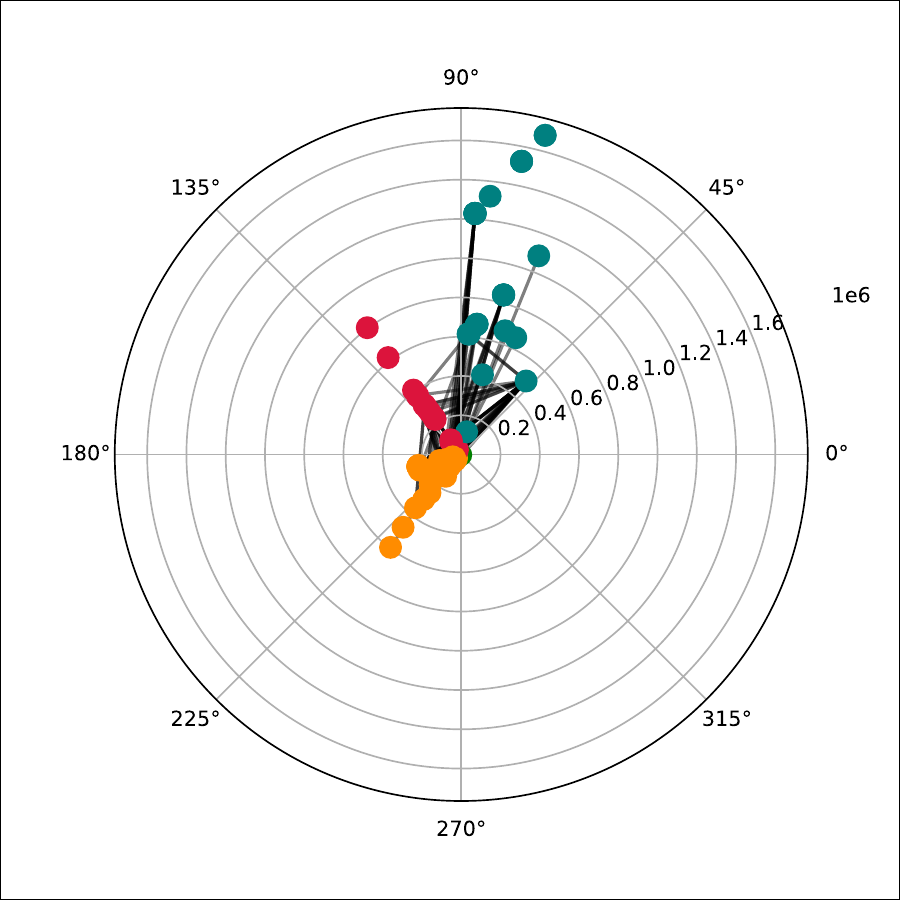} &
        \includegraphics[width=0.24\textwidth]{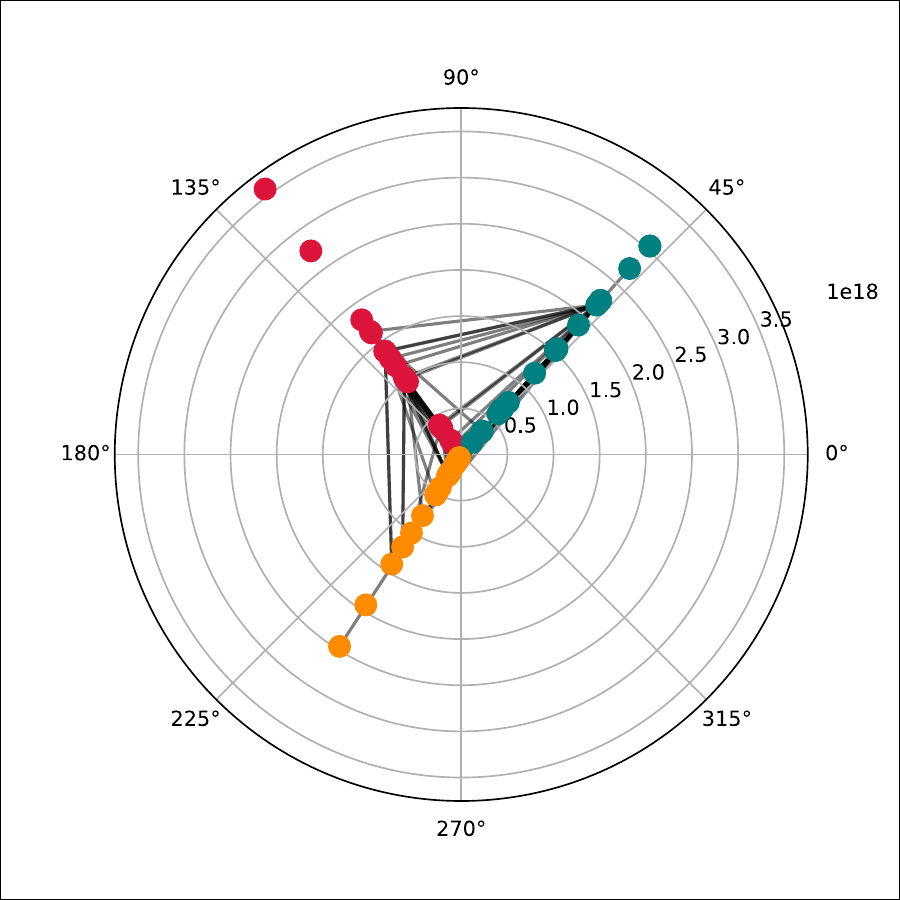}
    \end{tabular}
\label{fig:diffusion-left}
\end{subfigure}
\includegraphics[trim={0cm 2cm 0cm 2cm},clip,width=0.6\textwidth]{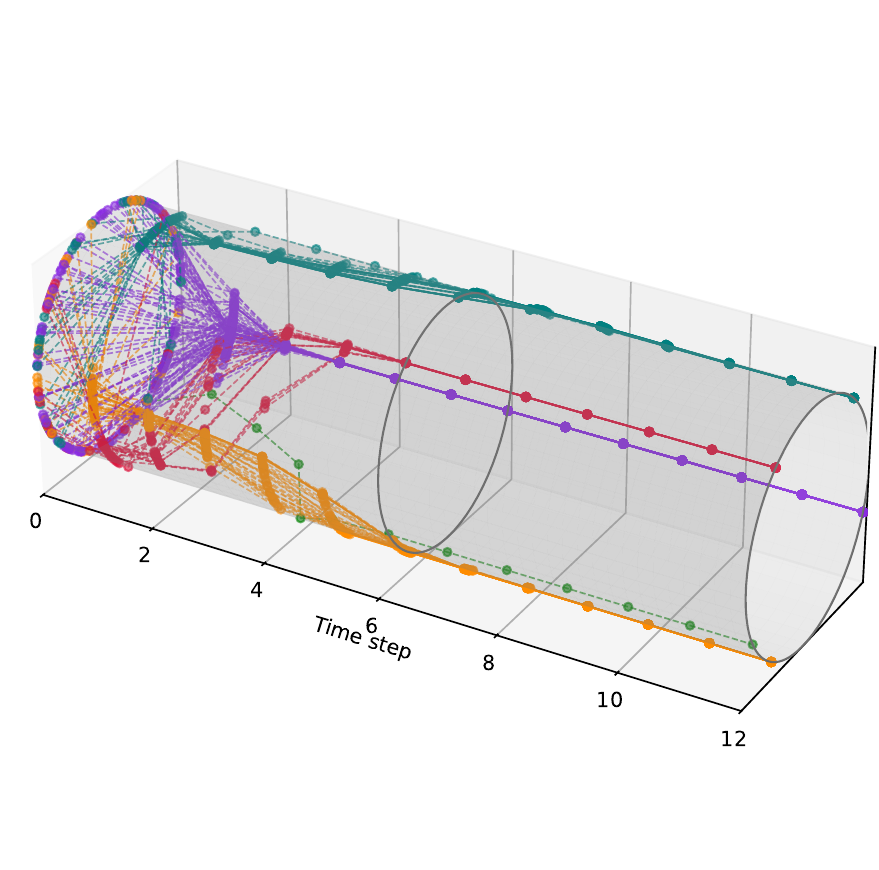}
\caption{Node features in polar coordinates at selected diffusion times (top) and evolution of node feature phases over time during diffusion (bottom).
    Complex random walk diffusion progressively separates the classes of a complex-weighted graph: nodes in the same class (same color) converge to complex numbers with the same phase. The figure corresponds to the dataset Texas \cite{geomgcn}. From diffusion time 5 onward, the purple class is omitted for better visualization. 
    }
    \label{fig:diffusion}
\end{figure}
\paragraph{Contributions.} We introduce a new approach that modifies the underlying “geometry” of the graph to define a more expressive diffusion process, forming the foundation for a novel GNN architecture. Concretely, we propose equipping the graph with a complex-weighted structure, where each edge is assigned a complex number. This extension enables a diffusion process that generalizes random walks to the complex domain \cite{tiancomplexnet}.

Our main contributions are as follows:
\begin{itemize}
\item We propose a novel way to enrich graph structure by assigning complex weights to edges, enabling a diffusion process with provable expressiveness guarantees. We show that for any node-classification task, suitable weights can be chosen such that classification is possible in the steady state of the resulting complex random walk (see Figure~\ref{fig:diffusion}). 
\new{To the best of our knowledge, this is the first study establishing a connection between the convergence properties of complex random walks and the expressivity of GNNs for node classification tasks.}
\item In contrast to sheaf-based diffusion, our theoretical guarantees are independent of the number of target classes, allowing for \new{a simpler model.}
\item We introduce the \textit{Complex-Weighted Convolutional Network} (CWCN), a GNN that augments complex random-walk diffusion and achieves greater expressiveness than GCNs. CWCN reduces the number of hyperparameters compared to prior methods while maintaining competitive performance on benchmarks.
\end{itemize}

\section{Theoretical Background}
\paragraph{Problem Setting}
We consider the problem of node classification in undirected graphs. Formally, let \( G = (V, E) \) be an undirected graph, with set of nodes  \( V = \{v_1, \dots, v_n\} \) and set of edges \( E \subset \{\{v_i, v_j\} \mid v_i, v_j \in V\} \). Each node \( v_i \in V \) is associated with a feature vector \( \mathbf{x}_i \in \mathbb{R}^k \) and belongs to a class \( y_i \in \{1, \dots, C\} \), where \( C \) is the number of possible classes. 
We denote by $\mb A\in\R^{n\times n}$ the adjacency matrix, $\mb D$ the degree matrix and $N(v_i)$ the set of neighbours of node $v_i$. In addition, we group all node feature vectors into a feature matrix $\mb X\in\R^{n\times k}$.

Given labels for a subset of nodes, the goal is to learn a function \mbox{\( f: V \to \{1, \dots, C\} \)} that predicts the class labels for the remaining nodes in the graph.
We focus on the semi-supervised node classification setting, where the complete graph structure -- including all node features and edge connections -- is available during training, but only a subset of nodes is labelled. 
In this setting, we denote by \( V_{\text{train}},V_{\text{val}},V_{\text{test}} \subset V \) the training, validation and test sets, respectively.

\subsection{Heat Diffusion and GCNs}
The heat diffusion equation in a graph and its unit-\new{timestep} Euler discretisation are
\begin{equation}\label{eq:heatdiff}
    \mb{\dot X}(t)=-\triangle_0\mb X(t) \quad \rightsquigarrow\quad \mb X(t+1)=(\mb I-\triangle_0)\mb X(t),
\end{equation}
where $\triangle_0:=\mb I-\mb{D^{-\frac{1}{2}}AD^{-\frac{1}{2}}}$ is the normalized graph Laplacian.

Most GNNs follow the message-passing paradigm, in which each node iteratively updates its feature representation by aggregating information from its neighbours. Numerous architectures are built upon this general framework, with Graph Convolutional Networks (GCNs) \cite{kipf2016semi} being one of the most popular examples. GCNs implement message passing at each layer as:
\begin{equation}
    \mb H^{(l)}=\sigma(\mb{\tilde D}^{-\frac{1}{2}}\mb{\tilde A}\mb{\tilde D}^{-\frac{1}{2}}\mb H^{(l-1)} \new{\mb M}_0^{(l)})=\sigma((\mb I-\tilde\triangle_0)\mb H^{(l)}\new{\mb M}_0^{(l)}),
\end{equation}
where $\mb H^{(l)}\in\R^{n\times k_l}$ is the matrix of node representations at layer $l$, $\mb \new{M}_0^{(l)}\in\R^{k_{l-1}\times k_{l}}$ is a learnable weight matrix, and $\sigma$ denotes a non-linear activation function. $\mb{\tilde A}=\mb A+\mb I$, $\mb{\tilde D}=\mb D+\mb I$ are the augmented adjacency and degree matrices, and $\tilde\triangle_0=\mb I-\mb{\tilde D^{-\frac{1}{2}}\tilde A\tilde D^{-\frac{1}{2}}}$ is the augmented graph Laplacian.
Therefore, GCN can be seen as an augmented heat diffusion process with an additional $k_{l-1}\times k_{l}$ learnable weight matrix $\new{\mb M}_0^{(l)}$ and a non-linearity $\sigma$. 

\citet{li2018deeper} were the first to highlight the oversmoothing problem, showing that repeatedly applying Equation~\eqref{eq:heatdiff} causes the features of nodes within the same connected component to converge to identical values. Building on this observation, \citet{oono2020graph} and \citet{cai2020note} demonstrate that, under certain assumptions on the weight matrices $\new{\mb M}_0^{(l)}$, the expressive power of GCNs decays exponentially as the number of layers increases.

\subsection{Random Walks on Graphs With Complex Weights}
In this section, we present the main concepts regarding complex-weighted graphs, which form the basis for studying the expressive power of the complex-weighted diffusion process, detailed in Section~\ref{sec:expressive-power}. 
\begin{definition}
    A complex-weighted graph is a graph where each edge is assigned a complex number, $G=(V,E,\mb{W})$ where $V=\{v_1,\dots,v_n\}$ is the node set, $E\subset\{e_{ij}=(v_i,v_j):v_i,v_j\in V\}$ is the edge set and $\mb{W}=(W_{ij})$ is the complex weight matrix, with $W_{ij}\in\C$ characterizing the edges between nodes.
\end{definition}

\begin{remark}
    We assume the complex-weighted graphs to be connected and directed\footnote{Note that we work with undirected unweighted graphs, but directed complex-weighted graphs.}. In addition, we assume $\mb{W}$ to be Hermitian, i.e., $\mb{W}=\mb{W}^*$.
\end{remark}

\begin{definition}\label{def:complex-matrices}
We define the following matrices in a complex-weighted graph:
\begin{enumerate}
    \item The complex degree matrix $\mb{Q}$ is the diagonal matrix with elements $q_i=\sum_j|W_{ij}|$.
    \item The complex transition matrix $\mb P=\mb{Q}^{-1}\mb{W}$.
    \item The complex random walk Laplacian $\mb{L}_{rw}=\mb{I}-\mb P$, where $\mb{I}$ is the identity matrix.
\end{enumerate}
\end{definition}
We can now use these matrices to define a complex random walk.
\begin{definition}\label{def:complex_rw}
    Let $G=(V,E,\mb{W})$ be a complex-weighted graph with $f$ feature channels, where node features are represented as a matrix $\mb{X}\in\C^{n\times f}$. We define a complex random walk as the diffusion process governed by the following PDE:
    \begin{equation}\label{eq:rw}
        \mb{\dot X}(t)=-\mb{L}_{rw}\mb{X}(t).
    \end{equation}
    Thus, its Euler discretization with a unit step is:
    \begin{equation}\label{eq:rw_eulerDiscr}
        \mb{X}(t+1)=(\mb{I}-\mb{L}_{rw})\mb{X}(t)=\mb{P}\mb{X}(t).
    \end{equation}
\end{definition}
An important class of complex-weighted graphs is that of structurally balanced graphs, where the sum of the edge phases along any cycle is an integer multiple of $2\pi$ \cite{Lange2015MagL,tiancomplexnet}. These graphs can be characterized in terms of a partition of the node set, and their asymptotic behavior under a complex random walk has been analyzed in the infinite-time limit \cite{tiancomplexnet}. Building on these results, we show in the next section that any node-classification task can be solved in the steady state of a complex random walk, provided the graph is equipped with a suitable complex-weighted structure defined by the task. We further prove that such weights always exist (see Theorem~\ref{cor:linearSep}), establishing a general expressiveness guarantee for our framework. Additional details on balanced graphs and their properties are provided in Appendix \ref{app:cw-graphs} or can be found in \cite{tiancomplexnet}.

\section{The Expressive Power of Complex-Weighted Diffusion} \label{sec:expressive-power}
\new{\citet{bodnar2022neural} showed that performing diffusion in a real-weighted graph is not sufficient to guarantee expressivity in an arbitrary classification task.} In this section, we prove the full expressive power of complex-weighted diffusion, which constitutes the main result of this paper.  

\begin{restatable}{theorem}{MainCorollary}\label{cor:linearSep}
    Let $G=(V,E)$ be an unweighted, undirected graph and $\mathcal{V}=\{V_i\}_{i=1}^{l_p}$ a partition of its nodes. Then, there exists a complex-weighted graph $G'=(V,E,\mb W)$ such that, starting from any initial features $\mb x(0)=(x_i(0))$, in the steady state of a complex random walk, the features of the nodes belonging to a subset $V_l$ have the same phase, and this phase is different for each subset. 
\end{restatable}
\begin{proof}[Proof sketch]
    \new{The proof builds upon the results on complex-weighted graphs developed by \citet{tiancomplexnet}, summarized in Appendix~\ref{app:cw-graphs}. According to Proposition~\ref{th:charact-balanced}, the partition $\mathcal{V}$ characterizes a structurally balanced graph satisfying the following properties: (i) any edge within each node subset in $\mathcal{V}$ has phase 0; (ii) any edge between the same pair of node subsets in $\mathcal{V}$ has the same phase; and (iii) if we define the graph $\tilde G$ considering each node subset in $\mathcal{V}$ as a super node, the phase of every cycle in $\tilde G$ is 0. In addition, Proposition~\ref{th:convergente-rw} characterizes the steady state of a complex random walk on any structurally balanced graph in terms of its associated partition.}

    \new{To prove the theorem, we first show that for any given partition $\mathcal{V}$, it is possible to assign complex weights to the edges of the graph $G$ so that conditions (i)–(iii) above are satisfied. The proof of this result proceeds by constructing a super graph $\tilde G$ that determines the complex weight corresponding to each pair of subsets of $\mathcal{V}$ such that (iii) holds. Specifically, we first build a cycle basis of $\tilde G$ composed of triangles and verify that each element of the basis satisfies (iii). We then extend this property to any other triangle and, consequently, to all cycles in $\tilde G$. This construction determines the desired complex-weighted graph $G'$.}

    \new{Finally, applying Proposition~\ref{th:convergente-rw} to the obtained graph $G'$ shows that, in the long time limit of a complex random walk, the features of nodes belonging to different subsets of $\mathcal{V}$ converge to complex numbers with different phases.}
    The full proof can be found in Appendix~\ref{app:cw-graphs}.
\end{proof}
Hence, as in Figure \ref{fig:diffusion}, every graph can be assigned complex weights such that in the steady state of a complex random walk, only nodes belonging to the same class have the same phase.
Then, the nodes' features are linearly separable in the asymptotic time limit, which shows that solving any node classification task can be reduced to performing diffusion with the right complex-weighted structure.

\section{Complex-Weighted Convolutional Networks}
Analogously to how GCN augments heat diffusion, we build a Complex-Weighted Convolutional Network (CWCN) that augments the complex random walk diffusion process. In addition, we propose a method to learn the complex-weighted structure from data. 
\subsection{Complex-Weighted Convolutions}\label{sec:convolutions}
Let $G$ be a complex-weighted graph with initial feature matrix $\mb X(0)\in\C^{n\times f}$,  complex weight matrix $\mb W\in\C^{n\times n}$, degree matrix $\mb Q\in \R^{n\times n}$ and recall $\mb P=\mb Q^{-1}\mb W$, $\mb L_{rw}=\mb I-\mb P$. We equip the diffusion process given by Equation~\eqref{eq:rw} with learnable weight matrices $\new{\mb M}_1\in\C$, $\new{\mb M}_2\in\C^{f\times f}$ and a non-linearity $\sigma$ to arrive at a model whose evolution is governed by:
\begin{equation}\label{cwcn:diffusion}
    \mb{\dot X}(t)=-\sigma(\mb L_{rw}(\mb I_n\otimes\new{\mb M}_1)\mb X(t)\new{\mb M}_2)=-\sigma( (\mb I-\mb P)(\mb I_n\otimes\new{\mb M}_1)\mb X(t)\new{\mb M}_2),
\end{equation}
where $f$ is the number of channels and $\otimes$ denotes the Kronecker product. $\new{\mb M}_1$ \new{applies a rotation to the node features in all the channels}, while $\new{\mb M}_2$ allows mixing each node's feature channels. Note that by setting $\new{\mb M}_1$, $\new{\mb M}_2$ and $\sigma$ to the identity, we recover the complex random walk diffusion equation. Therefore, the model is at least as expressive as complex random walk diffusion and can benefit from the linear separation power described in Theorem~\ref{cor:linearSep}. We focus on the time-discretised version of this model which allows us to use a new set of learnable weight matrices $\new{\mb M}_1^l\in\C, \new{\mb M}_2^l\in\C^{f\times f}$ at each layer $l=0,\dots,L-1$:
\begin{equation}\label{eq:layerOp}
    \mb{ X}(l+1)= \mb{ X}(l)-\sigma( (\mb I-\mb P)(\mb I_n\otimes\new{\mb M}_1^l)\mb X(l)\new{\mb M}_2^l)\in\C^{n\times f}.
\end{equation}
Since the formulation in Equation~\eqref{eq:layerOp} requires the initial feature matrix $\mb X(0)\in \C^{n\times f}$, we first use a multilayer perceptron (MLP) to map the raw node features to a real-valued matrix $\mb X(0)\in \R^{n\times 2f}$, which we then interpret as a complex matrix in $\C^{n\times f}$. Similarly, we interpret the final representation $\mb X(L)$ as a real-valued matrix in $\R^{n\times 2f}$ and apply a final MLP to perform node classification.

\subsection{Complex Weights Learning}\label{sec:cw_learning}
In general, an appropriate complex-weighted structure that guarantees linear separability in the time limit of the diffusion process cannot be determined without full knowledge of the node classes.
Therefore, we aim to learn the underlying complex structure from data, which will allow the model to choose the appropriate complex weight matrix $\mb W$ for each node classification task. Specifically, we define the complex weight matrix $\mb W$ as a learnable function of the initial node features $\mb X(0)$:
\begin{equation}
    \mb W=g(G,\mb X(0);\theta),
\end{equation}
where $\theta$ are learnable parameters.
Once we have learned the complex weight matrix $\mb W$, we can obtain $\mb Q$ the degree matrix and compute $\mb P=\mb Q^{-1}\mb W$. 

The complex weight matrix $\mb W \in\C^{n\times n}$ is learned using a parametric matrix-valued function. The weight corresponding to an edge $(v_i,v_j)$, $ W_{ij}$, is computed based on the initial features of nodes $v_i$ and $v_j$, which we denote by $\mb x_i,\mb x_j\in\C^{f}$. Formally, the weight matrix is given by:
\begin{equation}
    W_{ij}=\Phi (\mb x_i,\mb x_j),
\end{equation}
where $\Phi$ is a learnable function satisfying $\Phi (\mb x_i,\mb x_j)=\Phi (\mb x_j,\mb x_i)^*$, since $\mb W$ must be Hermitian. 

In practice, we set $\Phi$ to be an MLP and interpret $\mb x_i$ and $\mb x_j$ as real vectors in $\R^{k}$ with $k=2f$. For brevity, we use the same notation for both complex and real forms, as the intended interpretation will be clear from the context.
Given a pair of nodes $(v_i,v_j)$ with corresponding feature vectors $\mb x_i,\mb x_j\in\R^k$, we define $\Phi(\mb x_i,\mb x_j)$ as follows:
\begin{align}\label{eq:def-phi}
    \Phi(\mb x_i,\mb x_j)=\begin{cases}
        \tilde\sigma(\mb V[\mb x_i || \mb x_j]), & \text{if }(v_i,v_j)\in E\text{ and }i\leq j\\
        \overline{\tilde\sigma(\mb V[\mb x_i || \mb x_j]))}, & \text{if }(v_i,v_j)\in E\text{ and }i> j\\
        0 & \text{otherwise}
    \end{cases}
\end{align}
Here, $||$ denotes vector concatenation, $\mb V\in\R^{2\times 2k}$ is a learnable weight matrix and $\tilde\sigma$ is a non-linear function. The overline $\overline{(a,b)}$ denotes the complex conjugate when interpreting $(a,b)\in\R^2$ as a complex number, i.e., $\overline{(a,b)}=(a,-b)$. Finally, we interpret $\Phi(\mb x_i,\mb x_j)\in\R^2$ as a complex number.

The following result shows that if the function $\Phi$ has enough
capacity and the features are diverse enough, we can learn any Hermitian complex weight matrix.

\begin{restatable}{proposition}{PropLearnW}\label{prop:universality}
    Let $G=(V,E)$ be a finite graph with feature matrix $\mb X\in\C^{n\times f}$. If the node features $(\mb x_i,\mb x_j)\neq (\mb x_k,\mb x_t)$ for any $(v_i,v_j)\neq(v_k,v_t)\in E$ and $\Phi$ defined in \eqref{eq:def-phi} is an MLP with sufficient capacity, then $\Phi$ can learn any complex-weighted structure $G'=(V,E,\mb W)$.
\end{restatable}
\begin{proof}[Proof sketch]
    \new{Define the set $S:=\{(\mb x_i,\mb x_j): (v_i,v_j)\in E\}\subset\C^{2f}$ and the function $g\colon S\to\C$, $g(\mb x_i,\mb x_j)=W_{ij}$. Interpret $S$ as a subset of $\R^{2k}$, with $k=2f$, and $g\colon S\to\R^{2}$. First, we show that $g$ can be extended to a smooth function $f\colon\R^{2k}\to\R^{2}$. Then, $f$ can be approximated by an MLP with sufficient capacity \cite{hornik1991approximation,hornik1989multilayer} and, thus, so does $g$.
    Therefore, $\Phi$ defined in Equation~\eqref{eq:def-phi} can learn any complex weight matrix $\mb W$ of the graph $G$.}
    \new{For more details, see Appendix~\ref{app:cw-learning}.}
\end{proof}

\section{Experiments} 
Having shown theoretically the expressivity of our model, we evaluate it through both synthetic and real-world experiments. The synthetic experiments \new{show the benefits of using complex-weighted diffusion in controlled heterophilic settings}, while the real-world experiments assess performance on benchmark datasets against several baseline models, providing an overview of the model’s capabilities.
\subsection{Synthetic Experiments}\label{fig:synthetic_exp_main}
\begin{figure}[]
    \centering
    \begin{subfigure}[b]{0.49\textwidth}
        \centering
        \includegraphics[width=\linewidth]{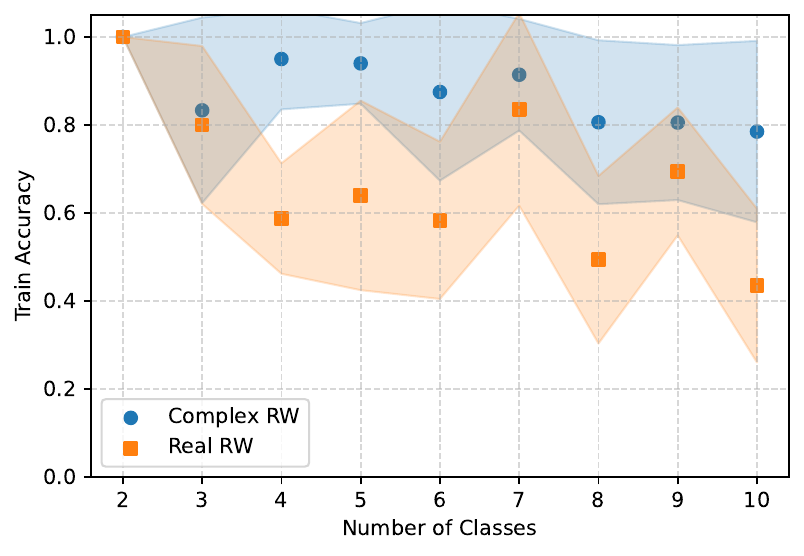}
    \end{subfigure}
    \hfill
    \begin{subfigure}[b]{0.49\textwidth}
        \centering
        \includegraphics[width=\linewidth]{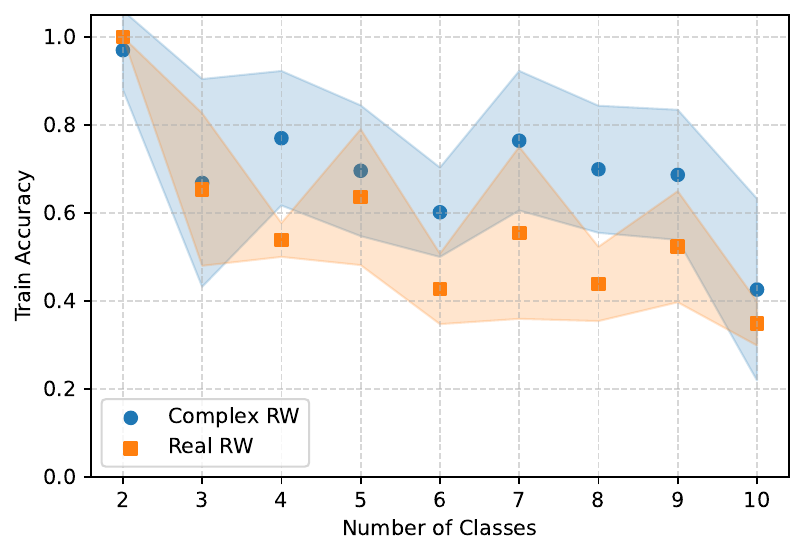}
    \end{subfigure}

    \vskip\baselineskip
    \begin{subfigure}[b]{0.45\textwidth}
        \centering
        \includegraphics[width=\linewidth]{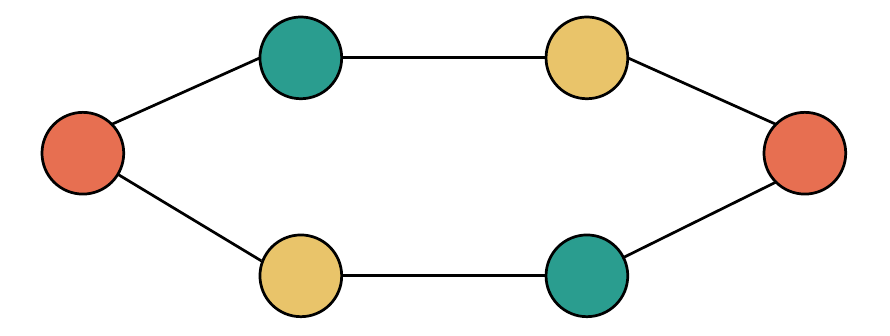}
    \end{subfigure}
    \hfill
    \begin{subfigure}[b]{0.45\textwidth}
        \centering
        \includegraphics[width=\linewidth]{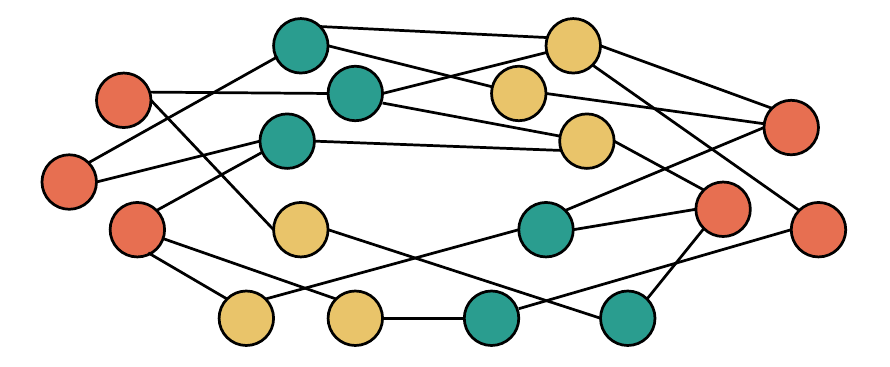}
    \end{subfigure}

    \caption{\new{Training accuracy (top) and examples of the processed graphs (bottom) for learnable complex versus real random walks, varying the numbers of classes in two heterophilic settings: a cycle graph with same-class nodes in opposite positions (left) and a ring of clusters with same-class clusters in opposite positions (right). The complex random walk consistently achieves higher mean accuracy (dots) than its real-valued counterpart (shade: standard deviation over 10 random seeds).}}
\end{figure}
\new{These experiments constitute an ablation study aimed at assessing whether complex-weighted diffusion provides advantages in heterophilic settings compared to its real-weighted counterpart, as suggested by our theoretical analysis. A real random walk can be defined analogously to Definitions~\ref{def:complex-matrices} and~\ref{def:complex_rw} (see \cite{tiansignednet} for details). While \citet{bodnar2022neural} demonstrate that real-valued diffusion fails to separate classes in the time limit for certain tasks, Theorem~\ref{cor:linearSep} establishes that complex diffusion can achieve linear separation in its steady state. To isolate the effect of complex weights, we evaluate learnable vanilla real and complex random walks by setting $\new{\mb M}_1 = 1$, $\new{\mb M}_2 = \mb I_f$, and $\sigma = \operatorname{id}$ in Equation~\eqref{eq:layerOp}, so that only the weight matrix $\mb W$ is trainable. To ensure a fair comparison, we use a single feature channel ($f=1$) for the complex case and two channels for the real case, ensuring that both models have the same number of learnable parameters. Finally, to analyse the asymptotic behaviour, we set the number of layers to 20.}

\new{We first consider a cycle graph containing two nodes per class, where nodes of the same class are placed at opposite positions on the cycle (Figure~\ref{fig:synthetic_exp_main}, bottom-left) and node features are sampled from distinct class-specific Gaussian distributions. This setup provides a clear example of heterophily: nodes of the same class are distant in the graph topology, and information must diffuse through an increasing number of intermediate nodes, proportional to the number of classes, before reaching another node of the same class.}

\new{Next, we design a more challenging heterophilic scenario. We consider a Stochastic Block Model (SBM) with 10 nodes per cluster and 2 clusters per class. The clusters are arranged in a ring topology, where edges are formed only between adjacent clusters, and no intra-cluster connections are present. Clusters positioned opposite to each other on the ring share the same class label and node features are sampled from class-specific Gaussian distributions. Figure~\ref{fig:synthetic_exp_main} (bottom-right) illustrates this configuration for the case of 3 nodes per cluster.}

\new{Figure~\ref{fig:synthetic_exp_main}(top) the classification accuracy for both settings as the number of classes increases, averaged over ten random seeds. Consistent with our theoretical findings, the complex random walk consistently outperforms its real-valued counterpart. Notably, in the two-class setting both methods achieve perfect accuracy, which aligns with the observations of \citet{bodnar2022neural}, who showed that real-weighted diffusion is sufficient when only two classes are present. Overall, these experiments empirically confirm the intrinsic advantage of complex-weighted diffusion in heterophilic graphs, in agreement with our theoretical results.}

\subsection{Real-World Experiments}
Given the conceptual similarity between our model and SCN \cite{bodnar2022neural}, a direct comparison is particularly relevant. To ensure fairness, we follow the same experimental procedure described in \cite{bodnar2022neural}.

\paragraph{Datasets.}
We evaluate our model on several real-world graphs \cite{citnet,geomgcn,wikipedianet}.  
These datasets exhibit varying degrees of edge homophily $h$, with values ranging from $h=0.11$ (very heterophilic) to $h=0.81$ (very homophilic). This diversity allows us to assess our model’s robustness under different homophily conditions.
Following \cite{bodnar2022neural}, we evaluate our model using the 10 fixed data splits provided by \cite{geomgcn}.  For each split, 48\% of the nodes in each class are used for training, 32\% for validation, and the remaining 20\% for testing. We report the mean accuracy and standard deviation across 10 splits.

\paragraph{Baselines.}
As baselines, we include SCN \cite{bodnar2022neural}, along with the same selection of Graph Neural Network (GNN) models used in their study. These baselines can be classified into three categories: (1) classical: GCN \cite{kipf2016semi}, GAT \cite{gat}, GraphSAGE \cite{graphsage}; (2) models designed for heterophilic settings: GGCN \cite{ggcn}, Geom-GCN \cite{geomgcn}, H2GCN \cite{zhu2020beyond}, GPRGNN \cite{GPRGNN}, FAGCN \cite{fagcn}, MixHop \cite{mixhop}; and (3) models designed to address oversmoothing: GCNII \cite{gcnii}, PairNorm \cite{zhao2020pairnorm}. All baseline results are reported as presented in \cite{bodnar2022neural}. For SCN \cite{bodnar2022neural}, we select O(d)-NSD, the variant that achieves the best average performance. Finally, we include two additional diffusion-based models: BLEND \cite{blend} and GREAD \cite{choi2023gread} (selecting the variant with best average performance). \cite{choi2023gread}.

 \paragraph{Results.}The results are summarized in Table~\ref{tab:experiments}. First, CWCN significantly outperforms classical GNNs on heterophilic datasets, supporting our theoretical claims regarding CWCN’s improved expressivity over GCNs and demonstrating that these advantages translate into practical performance gains. On homophilic datasets, our model also generally performs better, although the margins are smaller. Second, CWCN remains competitive across all datasets, with its performance deviating by an average of 3.52\% from the best-performing model (4.84\% for heterophilic datasets and 1.31\% for homophilic datasets). Overall, CWCN ranks $4^{th}$ in average empirical performance among all evaluated models while providing provable expressiveness guarantees for an infinite number of layers without additional constraints.
 
\definecolor{darkorange}{RGB}{151,110,0}
\renewcommand{\arraystretch}{1.2}
\begin{table*}
\caption{Accuracy results (mean test accuracy $\pm$ standard deviation) on node classification datasets, sorted by homophily level. The top four models are highlighted in \textcolor{red}{\textbf{First}}, \textcolor{blue}{\textbf{Second}}, \textcolor{violet}{\textbf{Third}}, \textcolor{darkorange}{\textbf{Fourth}}. The background color of the model name: {\setlength{\fboxsep}{0.2pt} \colorbox{green!40}{green}} for models that provide expressive power guarantees for an infinite number of layers, {\setlength{\fboxsep}{0.2pt} \colorbox{yellow!60}{yellow}} for models that provide them only under certain constraints, and {\setlength{\fboxsep}{0.2pt} \colorbox{gray!25}{grey}} for models without expressiveness guarantees. Our model is denoted as CWCN, and the gap to the best model is computed in \% as $\frac{\text{Acc}_{\text{bestModel}}-\text{Acc}_{\text{CWCN}} }{\text{Acc}_{\text{bestModel}}} \cdot 100$. Table adapted and modified from \cite{bodnar2022neural}.}
\label{tab:experiments}
\resizebox{\textwidth}{!}{%

\centering

\begin{tabular}{lcccccccccc}
\toprule
& \textbf{Texas} & \textbf{Wisconsin} & \textbf{Film} & \textbf{Chameleon} & \textbf{Cornell} & \textbf{Citeseer} & \textbf{Pubmed} & \textbf{Cora}& \textbf{Avg.}\\
\midrule
Hom level & \textbf{0.11} & \textbf{0.21} & \textbf{0.22} & \textbf{0.23} & \textbf{0.30} & \textbf{0.74} & \textbf{0.80} & \textbf{0.81} \\
\#Nodes & 183 & 251 & 7,600 & 2,277 & 183 & 3,327 & 18,717 & 2,708 &\\
\#Edges & 295 & 466 & 26,752 & 31,421 & 280 & 4,676 & 44,327 & 5,278& \\
\#Classes & 5 & 5 & 5 & 5 & 5 & 6 & 3 & 7& \\
\midrule
\cellcolor{green!40}\textbf{CWCN} & \textcolor{darkorange}{\textbf{84.05}}{\scriptsize $\pm$6.45} & \textcolor{darkorange}{\textbf{86.27}}{\scriptsize $\pm$4.20} & 36.51 {\scriptsize $\pm$1.26}& \textcolor{darkorange}{\textbf{65.59}} {\scriptsize $\pm$1.33} & 83.51{\scriptsize $\pm$8.15} &76.37{\scriptsize $\pm$1.53}  & 89.23 {\scriptsize $\pm$0.49}& \textcolor{darkorange}{\textbf{87.93}}{\scriptsize $\pm$1.03}&\textcolor{darkorange}{\textbf{76.18}} \\
\footnotesize{\shortstack{Gap to best\\model (\%)}}\rule{0pt}{4.5ex}
& 5.48 & 3.51 & 3.67 &  8.11  & 3.44 & 2.11 & 1.11&0.72&3.95\\

\midrule
\cellcolor{gray!25}GREAD & \textcolor{red}{\textbf{88.92}}{\scriptsize $\pm$3.72} & \textcolor{red}{\textbf{89.41}}{\scriptsize $\pm$3.30} & \textcolor{red}{\textbf{37.90}}{\scriptsize $\pm$1.17} & \textcolor{red}{\textbf{71.38}}{\scriptsize $\pm$1.31} & \textcolor{red}{\textbf{86.49}}{\scriptsize $\pm$7.15} & \textcolor{blue}{\textbf{77.60}}{\scriptsize $\pm$1.81} & \textcolor{red}{\textbf{90.23}}{\scriptsize $\pm$0.55} & \textcolor{red}{\textbf{88.57}}{\scriptsize $\pm$0.66}&\textcolor{red}{\textbf{78.81}}\\

\cellcolor{yellow!60}SCN & \textcolor{blue}{\textbf{85.95}}{\scriptsize $\pm$6.95} & \textcolor{red}{\textbf{89.41}}{\scriptsize $\pm$4.74} & \textcolor{blue}{\textbf{37.81}}{\scriptsize $\pm$1.15} & \textcolor{violet}{\textbf{68.04}}{\scriptsize $\pm$1.58} &  \textcolor{darkorange}{\textbf{84.86}}{\scriptsize $\pm$4.71}  & 76.70{\scriptsize $\pm$1.57}& \textcolor{darkorange}{\textbf{89.49}}{\scriptsize $\pm$0.40} & 86.90{\scriptsize $\pm$1.13}&  \textcolor{violet}{\textbf{77.39}}\\

\cellcolor{gray!25}BLEND & 83.24{\scriptsize $\pm$4.64} & 84.12{\scriptsize $\pm$3.56} & 35.63{\scriptsize $\pm$1.01} &  60.11{\scriptsize $\pm$2.09} & \textcolor{blue}{\textbf{85.95}}{\scriptsize $\pm$6.82} & 76.63{\scriptsize $\pm$1.60} & 89.24{\scriptsize $\pm$0.42} & 86.90{\scriptsize $\pm$1.13}& 75.23\\

\cellcolor{gray!25}GGCN & \textcolor{violet}{\textbf{84.86}}{\scriptsize $\pm$4.55} & \textcolor{violet}{\textbf{86.86}}{\scriptsize $\pm$3.29} &  \textcolor{violet}{\textbf{37.54}}{\scriptsize $\pm$1.56} & \textcolor{blue}{\textbf{71.14}}{\scriptsize $\pm$1.84} & \textcolor{violet}{\textbf{85.68}}{\scriptsize $\pm$6.63} &\textcolor{darkorange}{\textbf{77.14}}{\scriptsize $\pm$1.45} & 89.15{\scriptsize $\pm$0.37} & \textcolor{violet}{\textbf{87.95}}{\scriptsize $\pm$1.05}& \textcolor{blue}{\textbf{77.54}} \\

\cellcolor{gray!25}H2GCN & \textcolor{violet}{\textbf{84.86}}{\scriptsize $\pm$7.23} & \textcolor{blue}{\textbf{87.65}}{\scriptsize $\pm$4.98} & 35.7{\scriptsize $\pm$1.00} & 60.11{\scriptsize $\pm$2.15} & 82.70{\scriptsize $\pm$5.28} & 77.11{\scriptsize $\pm$1.57} & 88.49{\scriptsize $\pm$0.38} & 87.87{\scriptsize $\pm$1.20}& 75.56 \\

\cellcolor{gray!25}GPRGNN & 78.36{\scriptsize $\pm$4.31} & 82.94{\scriptsize $\pm$4.21} & 34.63{\scriptsize $\pm$1.22}  & 46.58{\scriptsize $\pm$1.71} & 80.27{\scriptsize $\pm$8.11} & 77.13{\scriptsize $\pm$1.67} & 87.54{\scriptsize $\pm$0.38} & \textcolor{violet}{\textbf{87.95}}{\scriptsize $\pm$1.18}& 71.92\\

\cellcolor{gray!25}FAGCN & 82.43{\scriptsize $\pm$6.89} & 82.94{\scriptsize $\pm$7.95} & 34.87{\scriptsize $\pm$1.25} & 55.22{\scriptsize $\pm$3.19} & 79.19{\scriptsize $\pm$9.79} & N/A & N/A & N/A& -\\

\cellcolor{gray!25}MixHop & 77.84{\scriptsize $\pm$7.73} & 75.88{\scriptsize $\pm$4.90} & 32.22{\scriptsize $\pm$2.34} &  60.50{\scriptsize $\pm$2.53} & 73.51{\scriptsize $\pm$6.34} & 76.26{\scriptsize $\pm$1.33} & 85.31{\scriptsize $\pm$0.61} & 87.61{\scriptsize $\pm$0.85}& 71.14\\

\cellcolor{gray!25}GCNII & 77.57{\scriptsize $\pm$3.83} & 80.39{\scriptsize $\pm$3.40} &  \textcolor{darkorange}{\textbf{37.44}}{\scriptsize $\pm$1.30} & 63.86{\scriptsize $\pm$3.04} & 77.86{\scriptsize $\pm$3.79} & \textcolor{violet}{\textbf{77.33}}{\scriptsize $\pm$1.48} & \textcolor{blue}{\textbf{90.15}}{\scriptsize $\pm$0.43} & \textcolor{blue}{\textbf{88.37}}{\scriptsize $\pm$1.25}& 74.12\\

\cellcolor{gray!25}Geom-GCN & 66.76{\scriptsize $\pm$2.72} & 64.51{\scriptsize $\pm$3.66} & 31.59{\scriptsize $\pm$1.15} &  60.00{\scriptsize $\pm$2.81} & 60.54{\scriptsize $\pm$3.67} & \textcolor{red}{\textbf{78.02}}{\scriptsize $\pm$1.15} & \textcolor{violet}{\textbf{89.95}}{\scriptsize $\pm$0.47} &  85.35{\scriptsize $\pm$1.57}&67.09\\

\cellcolor{gray!25}PairNorm & 60.27{\scriptsize $\pm$4.34} & 48.43{\scriptsize $\pm$6.14} & 27.40{\scriptsize $\pm$1.24} &  62.74{\scriptsize $\pm$2.82} & 58.92{\scriptsize $\pm$3.15} & 73.59{\scriptsize $\pm$1.47}  & 87.53{\scriptsize $\pm$0.44}  & 85.79{\scriptsize $\pm$1.01} & 63.08\\

\cellcolor{gray!25}GraphSAGE & 82.43{\scriptsize $\pm$6.14} & 81.18{\scriptsize $\pm$5.56} & 34.23{\scriptsize $\pm$0.99} & 58.73{\scriptsize $\pm$1.68} & 75.95{\scriptsize $\pm$5.01} & 76.04{\scriptsize $\pm$1.30} & 88.45{\scriptsize $\pm$0.50} & 86.90{\scriptsize $\pm$1.04}& 72.99 \\

\cellcolor{gray!25}GCN & 55.14{\scriptsize $\pm$5.16} & 51.76{\scriptsize $\pm$3.06} & 27.32{\scriptsize $\pm$1.10} &  64.82{\scriptsize $\pm$2.24} & 60.54{\scriptsize $\pm$5.30} & 76.50{\scriptsize $\pm$1.36} & 88.42{\scriptsize $\pm$0.50} & 86.98{\scriptsize $\pm$1.27}& 63.93\\

\cellcolor{gray!25}GAT & 52.16{\scriptsize $\pm$6.63} & 49.41{\scriptsize $\pm$4.09} & 27.44{\scriptsize $\pm$0.89} & 60.26{\scriptsize $\pm$2.50} & 61.89{\scriptsize $\pm$5.05} & 76.55{\scriptsize $\pm$1.23} & 87.30{\scriptsize $\pm$1.10} & 86.33{\scriptsize $\pm$0.48}& 62.67\\

\bottomrule
\end{tabular}
}

\end{table*}

\section{Discussion and Related Work}

\paragraph{Heterophily and Oversmoothing.} Heterophilic graphs challenge the homophily assumption underlying many GNNs. To address this, several strategies have been proposed. MixHop \cite{mixhop} aggregates information from higher-order neighbourhoods to capture long-range dependencies; Geom-GCN \cite{geomgcn} redefines the notion of neighbourhood; FAGCN \cite{fagcn}, H2GCN \cite{zhu2020beyond} and GGCN \cite{ggcn}  model the relative importance of neighbours during aggregation; and  GPRGNN \cite{GPRGNN} integrates representations from multiple layers to jointly leverage local and global structural information.

To mitigate oversmoothing, a variety of methods have been proposed. Architecture-agnostic techniques include residual or skip connections to preserve information flow across layers \cite{li2019deepgcns, xu2018jumpknow}, normalization methods to limit feature homogenization \cite{zhao2020pairnorm}, and graph rewiring to enhance connectivity \cite{rong2019dropedge}. Architectures such as GCNII \cite{gcnii} and PairNorm \cite{zhao2020pairnorm} exemplify these approaches.

 These methods offer practical mechanisms to address heterophily and oversmoothing. While many include theoretical analyses that highlight the models’ advantages, they generally lack theoretical guarantees regarding node features separability as the number of layers increases. In contrast, CWCN not only ensures such guarantees but also achieves superior empirical performance in heterophilic settings compared to most of these models, being only slightly outperformed by GGCN \cite{ggcn}.
 
 \paragraph{Diffusion on GNNs.} More recently, several approaches jointly address oversmoothing and heterophily by modifying the underlying message-passing dynamics \cite{bodnar2022neural, maskey2023fractional, choi2023gread,blend}. A notable example is the Sheaf Convolutional Network (SCN), introduced by \citet{hansen2020sheaf} and later extended into a practical learning framework by \citet{bodnar2022neural}. SCNs increase the expressive power of heat diffusion by equipping the graph with a \textit{cellular sheaf} \cite{curry2014sheaves}, enabling a diffusion process based on the sheaf Laplacian. In this setup, the sheaf structure is learned from data, and sheaf diffusion is augmented to build a GNN architecture. Another prominent approach is the Graph Neural Reaction-Diffusion Network (GREAD) \cite{choi2023gread}, which models feature propagation through reaction-diffusion equations. Since our method also introduces a novel diffusion process to redefine message passing, it naturally belongs to this family.
 
 While GREAD achieves state-of-the-art results on standard node-classification benchmarks, it does not provide formal theoretical guarantees. SCNs, in contrast, offer provable expressiveness: \citet{bodnar2022neural} show that, in the time limit of sheaf diffusion, any node classification task can theoretically be solved, provided the graph is equipped with an appropriate sheaf structure. However, these guarantees rely on the sheaf dimension scaling with the number of target classes. Without such scaling, sheaf diffusion only guarantees linear separation power for regular graphs. Importantly, increasing the sheaf dimension introduces additional learnable parameters and enlarges the diffusion matrix, leading to higher computational cost.
Furthermore, compared to standard GNNs, GREAD requires multiple additional hyperparameters to parameterize the reaction–diffusion process, while SCNs add only the sheaf dimension as hyperparameter.
 
 \paragraph{Advantages of CWCNs.} CWCN achieves competitive performance on node-classification benchmarks, surpassing most existing architectures designed to mitigate heterophily and oversmoothing, while performing only slightly below top-performing diffusion-based models. Importantly, unlike other architectures, CWCN achieves linear separation power in the diffusion time limit using a fixed-size matrix that is independent of the number of classes. In addition, it introduces no extra hyperparameters beyond those of standard GNNs. Thus, CWCN provides a simpler formulation that requires less hyperparameter tuning while offering stronger theoretical guarantees.

\paragraph{Limitations of CWCNs.} Despite its theoretical strengths and simplicity, CWCN has two main limitations. First, there remains an empirical performance gap between CWCN and the best-performing methods such as GREAD and SCN. Second, CWCN relies on complex-valued weights, which require additional matrix operations during both training and inference. This introduces overhead in backpropagation and leads to higher computational cost compared to real-valued GNN architectures. 

\new{\paragraph{Complex GNNs.} Incorporating complex values into graph learning has recently attracted attention from various communities, but the existing approaches differ fundamentally from our framework. For instance, CayleyNets \cite{levie2019cayley} extend the ChebNet paradigm by employing Cayley polynomials as spectral filters, which involve learnable complex coefficients; however, the resulting filters remain real-valued, and signals remain consistently real throughout the process. More broadly, in network science, there has been a growing interest in systematically extending classical concepts to complex-weighted networks, e.g., \cite{bottcher2024complex}. This line provides useful foundations but does not design GNN models for downstream applications. Closer to GNNs, MagNet and related works \cite{zhang2021magnet,xu2025multihop} employ a complex Laplacian matrix (e.g., magnetic Laplacian) to capture directed edge information, typically with one or two global phase parameters. In contrast, our framework allows arbitrary complex phases per edge, leading to a richer diffusion process. More importantly, unlike prior studies, we provide theoretical guarantees on the expressivity of our complex-weighted diffusion process. 
This combination of flexible modeling and rigorous theory distinguishes our contribution.
}

\section{Conclusion and Future Work}
In this work, we introduced a novel framework for enhancing the expressive power of GNNs through diffusion on complex-weighted graphs. We first established the full expressive power of complex-weighted diffusion, demonstrating its potential to address two well-known limitations of standard GNN architectures: oversmoothing and poor performance on heterophilic graphs.

Building on this theoretical insight, we proposed CWCN, a GNN architecture that augments complex random-walk diffusion with learnable parameters and nonlinear activations. We further introduced a mechanism to learn the complex-weighted structure directly from data, allowing the model to adaptively capture the most suitable ``geometry'' for a given task. Compared to prior methods, the resulting framework is simpler—requiring fewer hyperparameters—while being supported by stronger theoretical guarantees.

Empirical evaluations on standard node-classification benchmarks show that CWCN achieves competitive performance, although further work is needed to determine whether our model can be refined to close the gap with top-performing methods. Promising directions for future work include investigating alternative architectures based on complex random walks, exploring other types of diffusion processes on complex-weighted graphs, incorporating a reaction term into the complex-weighted diffusion equation or, interpreting complex multiplication as a rotation in two dimension, considering transformations along the walks in higher dimensions \cite{tian2025matrix}. Additionally, reducing the computational overhead associated with complex weights is an important avenue for improvement.

Overall, our findings suggest that incorporating complex-weighted
diffusion provides a powerful approach to designing more expressive GNNs. By enriching the message passing dynamics with a complex-weighted structure, we open new possibilities for designing models that are both theoretically grounded and practically effective. 
To the best of our knowledge, this is the first work to leverage complex
weights to enhance GNN expressiveness, and we hope it inspires further
exploration of their potential in graph-based learning tasks.

\section*{Acknowledgements}
C.L.A.\ acknowledges that the project that gave rise to these results received the support of a fellowship from the “la Caixa” Foundation (ID 100010434). The fellowship code is LCF/BQ/EU24/12060069. T.S.\ acknowledges financial support from the Rhodes Trust and the EPSRC Centre for Doctoral Training in Mathematics of Random Systems: Analysis, Modelling and Simulation (ESPRC Grant EP/S023925/1). R.L.\ acknowledges support from the EPSRC grants EP/V013068/1, EP/V03474X/1 and EP/Y028872/1.

\bibliographystyle{unsrtnat}
\bibliography{reference}

\newpage 
\appendix
\section{Complex Random Walk Proofs}\label{app:cw-graphs}
In this Appendix, we first summarize the main results of \cite{tiancomplexnet} on balanced graphs and their behavior under complex random walks. We then use these results to show that any node classification task can be performed in the time limit of a complex random walk, if the graph is equipped with a suitable complex-weighted structure.

\subsection{Complex Random Walks on Balanced Graphs}
\begin{notation}
    We express elements in the complex weight matrix in polar coordinates as $W_{ij}=r_{ij}e^{i\varphi_{ij}}$, where $r_{ij}\geq 0$ indicates the magnitude and $\varphi_{ij}\in[0,2\pi)$ is the phase. Thus, since $\mathbf{W}$ is Hermitian, $r_{ij}=r_{ji}$ and $\varphi_{ij}=-\varphi_{ji}+2\pi$.
\end{notation}
Next, we define the notion of structural balance in a complex-weighted graph and present necessary and sufficient conditions under which a graph is structurally balanced. To this end, we first introduce the concept of the phase of a path.
\begin{definition}
    Let $G=(V,E,\mathbf{W})$ be a complex-weighted graph and $P=(e_1,\dots,e_k),\quad e_i\in E$ be a path, the phase of $P$  is:
    \[\theta(P):=\sum_{i=1}^k\theta(e_i) \operatorname{mod} 2\pi,\]
    where $\theta(e)$ returns the phase of edge $e$.
\end{definition}
\begin{definition}
    A complex-weighted graph $G=(V,E,\mathbf{W})$ is structurally balanced if the phase of every cycle is 0.
\end{definition}
\begin{remark}
    Since $\mb{W}$ is Hermitian, if a cycle has phase $\theta$, the cycle with reverse direction has phase $2\pi-\theta$. Then, a cycle has phase $0$ iff its reversed cycle has phase 0, thus structural balance is well-defined.
\end{remark}
Next, we state a characterization of structurally balanced graphs, which is illustrated in Figure \ref{fig:balancedGraph}. 
\begin{figure}
    \centering
    \begin{tikzpicture}
    \tikzset{every path/.style={->}}
    \tikzstyle{group1}=[circle, draw=black, fill=Apricot, minimum size=4pt]
    \tikzstyle{group2}=[circle, draw=black, fill=Thistle, minimum size=4pt]
    \tikzstyle{group3}=[circle, draw=black, fill=SkyBlue, minimum size=4pt]

    \node[group1] (A) at (-5, 5.5) {A};
    \node[group1] (B) at (-1, 4.4) {B};
    \node[group1] (C) at (1, 6.5) {C};
    
    \node[group2] (D) at (2, 3.5) {D};
    \node[group2] (E) at (5, 3.5) {E};
    \node[group2] (F) at (5, 5.5) {F};
    
    \node[group3] (G) at (-5, 1.2) {G};
    \node[group3] (H) at (5, 1.2) {H};

    \draw (B) [bend right=10] to  node[pos=0.7, below] {$\scriptstyle 1$} (C);
     \draw (C) [bend right=10] to  node[pos=0.7, above] {$\scriptstyle 1$} (B);
    \draw (C) [bend right=10] to  node[pos=0.5, above] {$\scriptstyle 1$} (A);
    \draw (A) [bend right=10] to  node[pos=0.5, below] {$\scriptstyle 1$} (C);
    
    \draw (E) [bend right=10] to  node[pos=0.5, above] {$\scriptstyle 1$}(D);
    \draw (D) [bend right=10] to  node[pos=0.5, below] {$\scriptstyle 1$}(E);
    \draw (F) [bend right=10] to node[pos=0.6, above] {$\scriptstyle 1$}(D);
    \draw (D) [bend right=10] to node[pos=0.6, below] {$\scriptstyle 1$}(F);
    
    \draw (C) [bend right=10] to node[pos=0.4, below] {$e^{\frac{3\pi i}{7}}$}(F);
    \draw (F) [bend right=10] to node[pos=0.6, above] {$e^{\frac{-3\pi i}{7}}$}(C);
    
    \draw[bend right=20] (A) to node[pos=0.4, left] {$e^{\frac{\pi i}{7}}$}(G);
    \draw[bend right=20] (G) to node[pos=0.7, right] {$e^{-\frac{\pi i}{7}}$}(A);
    \draw (G)[bend right=10] to node[pos=0.8, below] {$e^{-\frac{\pi i}{7}}$}(B);
    \draw (B)[bend right=10] to node[pos=0.4, above] {$e^{\frac{\pi i}{7}}$}(G);
    \draw (D) [bend right=5] to node[pos=0.2, above] {$e^{-\frac{2\pi i}{7}}$}(G);
    \draw (G) [bend right=5] to node[pos=0.8, below] {$e^{\frac{2\pi i}{7}}$}(D);
    
    \draw (H) [bend right=20] to node[pos=0.5, right] {$e^{\frac{2\pi i}{7}}$} (E);
    \draw (E) [bend right=20] to node[pos=0.5, left] {$e^{\frac{-2\pi i}{7}}$} (H);
    \draw (G) [bend right=5] to node[pos=0.5, below] {$\scriptstyle 1$}(H);
    \draw (H) [bend right=5] to node[pos=0.5, above] {$\scriptstyle 1$}(G);

    \draw (E) [bend right=10] to node[pos=0.5, right] {$\scriptstyle 1$} (F);
    \draw (F) [bend right=10] to node[pos=0.5, left] {$\scriptstyle 1$} (E);

\end{tikzpicture}
\caption{Example of a balanced complex-weighted graph, where the node set is partitioned into three subsets satisfying the conditions of Proposition~\ref{th:charact-balanced}.}
\label{fig:balancedGraph}
\end{figure}
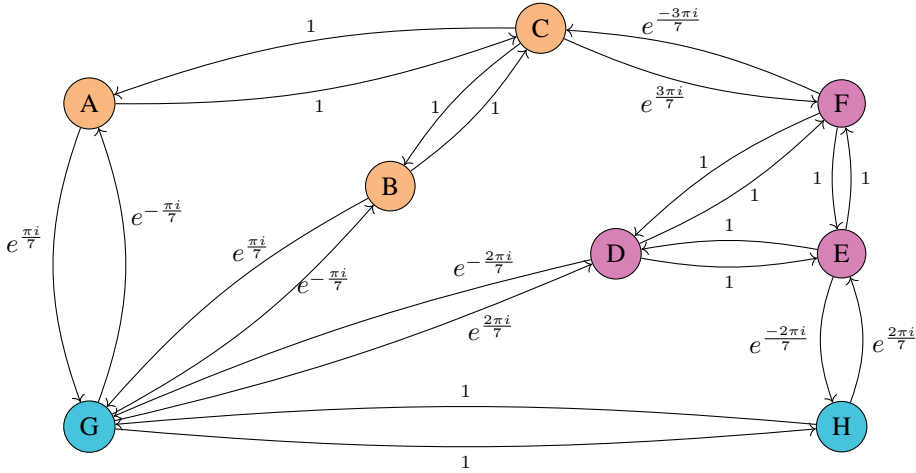
\begin{proposition}\label{th:charact-balanced}
    A complex-weighted graph $G$ is balanced if and only if there is a partition of the nodes $\mathcal{V}=\{V_i\}_{i=1}^{l_p}$ such that:
    \begin{enumerate}[label=\roman*)]
        \item Any edges within each node subset in $\mathcal{V}$ have phase 0.
        \item Any edges between the same pair of node subsets in $\mathcal{V}$ have the same phase.
        \item If we define the graph $G'$ considering each node subset in $\mathcal{V}$ as a super node, the phase of every cycle in $G'$ is 0.
    \end{enumerate}
\end{proposition}
\begin{proof}
    See \cite{tiancomplexnet}.
\end{proof}
The following result characterizes the asymptotic behaviour of a complex random walk on a balanced graph within the infinite time limit. Note that we state the proposition for one feature channel for simplicity, and the asymptotic behaviour of $\mb{X}\in\C^{n\times f}$ is defined by considering the behaviour of each channel separately.

\begin{proposition}\label{th:convergente-rw} 
Let $G$ be a balanced and not bipartite complex-weighted graph with associated partition $\mathcal{V}=\{V_i\}_{i=1}^{l_p}$ (Proposition~\ref{th:charact-balanced}). Then, the steady state of a complex random walk is $\mb {\hat x}=(\hat x_j)$, with
\begin{equation}\label{eq:steadystate}
\hat x_j=e^{i\theta_{1\sigma(j)}}\frac{\mb x(0)^T\tilde{\mathbf{1}}_1d_j}{2m},     
\end{equation}
where:
\begin{itemize}
    \item $\mb x(0)=(x_j(0))\in\C^n$ is the initial state vector.
    \item $2m=\sum_{j=1}^nd_j$.
    \item $\theta_{hl}$ is the phase of a path from nodes in $V_h$ to nodes in $V_l$.
    \item $\sigma(\,\cdot\,)$ returns the index of the node subset in $\mathcal{V}$ that a node is associated with.
    \item $\tilde{\mathbf{1}}_1$ is the diagonal vector of $\mathbf{S}^*$, where $\mathbf{S}$ is the diagonal matrix whose $(i,i)$ element is $e^{\theta_{1\sigma(i)}i}$.
\end{itemize}

\end{proposition}
\begin{proof}
    See \cite{tiancomplexnet}.
\end{proof}
\subsection{The Linear Separation Power of Complex-Weighted
Diffusion}
Let us first prove that every graph can be endowed with an appropriate complex-weighted structure.

\begin{proposition}\label{th:existence-balanced}
        Let $G=(V,E)$ be an unweighted, undirected graph and $\mathcal{V}=\{V_i\}_{i=1}^{l_p}$ a partition of its nodes. Then, there exists a complex-weighted graph $G'=(V,E, \mb W)$ such that:
        \begin{enumerate}
            \item It satisfies conditions (i), (ii) and (iii) of Proposition \ref{th:charact-balanced} for the partition $\{V_i\}_{i=1}^{l_p}$. Therefore, it is a balanced graph.
            \item All the edges of condition (ii) between the subsets $V_1$ and $V_i$ have different weight for each $i$.

        \end{enumerate}
\end{proposition}
\begin{proof}
    Let us first consider the graph with $l_p$ nodes resulting from considering each node subset $V_i$ of $G$ as a super node. We denote this graph by $\tilde{G}$ and its nodes by $\tilde{v}_1,\dots,\tilde{v}_{l_p}$, where each $\tilde{v}_i$ corresponds to the subset $V_i$.
    We will show that it is possible to assign weights to $\tilde G$ satisfying (iii) and such that the edges between nodes $\tilde v_1$ and $\tilde v_i$ are different for each $i$. We denote by $\mb{\tilde W}=(\tilde W_{ij})$ the complex weight matrix of $\tilde G$.
    
    We assume w.l.o.g. that $\tilde G$ is complete (if condition (iii) holds for $\tilde G$ complete, it holds for any graph with $l_p$ nodes, since removing edges does not add any new cycle to the graph).

    In addition, for any complete graph it is possible to choose a cycle basis whose elements are all triangles. To see that, note that a cycle basis can be obtained from any spanning tree of $\tilde G$ by selecting the cycles formed by combining a path in the tree with a single edge outside the tree.
    Therefore, we can choose the fundamental cycle basis formed from the spanning tree with edges $\{(\tilde v_1,\tilde{v}_2),(\tilde{v}_1,\tilde{v}_3),\dots,(\tilde{v}_1,\tilde{v}_{l_p})\}$. Denote this fundamental cycle basis by $\{T_1,\dots,T_m\}$ and note that it contains every triangle of the form $(\tilde v_{i},\tilde v_{j},\tilde v_1),\,i\neq j$.

   We assign complex weights to $\tilde G$ in the following way.
   First, set $\tilde W_{1i}=e^{i\theta_i}$, choosing $\theta_i$ such that $e^{i\theta_i}\neq e^{i\theta_j}$ for all $i\neq j$. Then, set $\tilde W_{i1}=\overline{\tilde{W}_{i1}}=e^{-i\theta_i}$. Finally, for $k=1,\dots,m$, assign the weights $\tilde W_{ij}$ and $\tilde W_{ji}$ of the remaining edge $e_k$ of each $T_k$ so that the sum of the phases of the cycle $T_k$ is $0$. Note that $\tilde W_{ji}=\overline{\tilde W_{ij}}$ trivially. Note that two triangles of the basis cannot share the same edge $e_k$ because of the way the fundamental basis is built, so this is well-defined. 
    
    Since $\tilde G$ is complete, every edge in $\tilde G$ belongs to some $T_k$. Therefore, we have assigned weights to every edge in $\tilde G$ such that the edges between nodes $\tilde v_1$ and $\tilde v_j$ are different for each $j$ and (iii) is satisfied for the cycles $T_1,\dots, T_m$. Next, we will prove that (iii) holds for every other triangle and, finally, that it holds for any cycle in $\tilde G$.

    Consider first a triangle $T=(\tilde v_{i_1}, \tilde v_{i_2}, \tilde v_{i_3})$ in $\tilde G$, such that $T\notin\{T_1,\dots, T_m\}$, thus $\tilde v_{i_j}\neq \tilde v_1$ for $j=1,2,3$. Define the triangles $T'_1:=(\tilde v_{i_1}, \tilde v_{i_2}, \tilde v_1)$, $T'_2:=(\tilde v_{i_2}, \tilde v_{i_3}, \tilde v_1)$, $T'_3:=(\tilde v_{i_3}, \tilde v_{i_1}, \tilde v_1)$ and the cycle $C=(\tilde v_{i_1},\tilde v_{i_2},\tilde v_{i_3},\tilde v_1)$. Note that $T_u'\in\{T_1,\dots, T_m\}$ for $u=1,2,3$ and that:
    \[T_1'\triangle T_2'=C\quad\quad\quad C\triangle T_3'=T,\]
    where $\triangle$ is the symmetric difference. Then, $T$ is expressed as a symmetric difference of basis triangles as $T=T_1'\triangle T_2' \triangle T_3'$. This is illustrated in Figure \ref{fig:symDiff}. Fix the orientation in $T$ given by the order $(\tilde v_{i_1}, \tilde v_{i_2}, \tilde v_{i_3})$ and denote by $\theta_{12}, \theta_{23}, \theta_{31}$ the phases corresponding to the weights of this oriented cycle. Next, set the orientations of $T'_1$, $T'_2$ and $T'_3$ given by the orders $(\tilde v_{i_1},\tilde v_{i_2},\tilde v_1)$, $(\tilde v_{i_2},\tilde v_{i_3},\tilde v_1)$ and $(\tilde v_{i_3},\tilde v_{i_1},\tilde v_1)$, respectively, as illustrated in Figure~\ref{fig:symDiff}.

\begin{figure}[]
\begin{subfigure}[]{\textwidth}
\centering
\begin{tikzpicture}[>=Latex, scale=1,
  every node/.style={
    circle,
    draw,
    minimum size=0.8cm,
    inner sep=1pt,
    font=\scriptsize
  }
]

\node[draw=none, minimum size=0] at (-1.1,0) {$\mb T$};
\node[draw=none, minimum size=0,text=orange] at (0.6,-0.7) {$T_2'$};
\node[draw=none, minimum size=0,text=purple] at (0.6,0.7) {$T_1'$};
\node[draw=none, minimum size=0,text=teal] at (-2.5,0) {$T_3'$};

\node (v1) at (3,0) {$\tilde{v}_1$};
\node (vi1) at (-2,2) {$\tilde{v}_{i_1}$};
\node (vi2) at (0.5,0) {$\tilde{v}_{i_2}$};
\node (vi3) at (-2,-2) {$\tilde{v}_{i_3}$};

\draw[->, purple, thick] ([yshift=2pt]vi2.east) -- ([yshift=2pt]v1.west);
\draw[->, orange, thick]  ([yshift=-2pt]v1.west) -- ([yshift=-2pt]vi2.east);

\draw[->, teal, thick] ([xshift=-2pt]vi3.north) -- ([xshift=-2pt]vi1.south);
\draw[->, black, very thick] ([xshift=2pt]vi3.north) -- ([xshift=2pt]vi1.south);

\draw[->, black,very thick] ([xshift=-2pt, yshift=-2pt]vi1.south east) -- ([xshift=-2pt, yshift=-2pt]vi2.north west);
\draw[->, purple, thick] ([xshift=2pt, yshift=2pt]vi1.south east) -- ([xshift=2pt, yshift=2pt]vi2.north west);

\draw[->, black, very thick] ([xshift=-2pt, yshift=2pt]vi2.south west) -- ([xshift=-2pt, yshift=2pt]vi3.north east);
\draw[->, orange, thick] ([xshift=2pt, yshift=-2pt]vi2.south west) -- ([xshift=2pt, yshift=-2pt]vi3.north east);

\draw[->, teal, thick] ([xshift=-2pt, yshift=2pt]vi1.north east) -- ([xshift=2pt, yshift=2pt]v1.north west);
\draw[->, purple, thick] ([xshift=-2pt, yshift=-2pt]v1.north west) -- ([xshift=4pt, yshift=-7pt]vi1.north east);

\draw[->, orange, thick] ([xshift=4pt, yshift=8pt]vi3.south east)--([xshift=-2pt, yshift=2pt]v1.south west) ;
\draw[->, teal, thick] ([xshift=2pt, yshift=-1pt]v1.south west) -- ([xshift=3pt, yshift=2pt]vi3.south east);
\end{tikzpicture}
  \end{subfigure}

\begin{subfigure}[]{\textwidth}
\centering
\begin{tikzpicture}[>=Latex, scale=0.8,
  every node/.style={
    circle,
    draw,
    minimum size=0.8cm,
    inner sep=1pt,
    font=\scriptsize
  }
]

\node[draw=none, minimum size=0] at (-2.3,-2.8) {$\mb T$};
\node[draw=none, minimum size=0,text=orange] at (-12.8,-2) {$T_2'$};
\node[draw=none, minimum size=0,text=purple] at (-12.8,0.75) {$T_1'$};
\node[draw=none, minimum size=0,text=teal] at (-7.25,-5) {$T_3'$};
\node[draw=none, minimum size=0,text=black] at (-6,-0.6) {$C$};

\node (v1) at (-10,0) {$\tilde{v}_1$};
\node (vi1) at (-15,2) {$\tilde{v}_{i_1}$};
\node (vi2) at (-12.5,0) {$\tilde{v}_{i_2}$};

\node (v1_2) at (-10,-1.2) {$\tilde{v}_1$};
\node (vi2_2) at (-12.5,-1.2) {$\tilde{v}_{i_2}$};
\node (vi3_2) at (-15,-3.2) {$\tilde{v}_{i_3}$};

\node (v1_3) at (-5,-0.6) {$\tilde{v}_1$};
\node (vi1_3) at (-8.5,1) {$\tilde{v}_{i_1}$};
\node (vi2_3) at (-7,-0.6) {$\tilde{v}_{i_2}$};
\node (vi3_3) at (-8.5,-2.2) {$\tilde{v}_{i_3}$};

\node (v1_4) at (-5,-5) {$\tilde{v}_1$};
\node (vi1_4) at (-8.5,-3.4) {$\tilde{v}_{i_1}$};
\node (vi3_4) at (-8.5,-6.6) {$\tilde{v}_{i_3}$};

\node (vi1_5) at (-3,-0.9) {$\tilde{v}_{i_1}$};
\node (vi2_5) at (-1,-2.8) {$\tilde{v}_{i_2}$};
\node (vi3_5) at (-3,-4.7) {$\tilde{v}_{i_3}$};

\draw[->] (-9.5,-0.6) -- (-8.2,-0.6) node[midway, above, draw=none, shape=rectangle, inner sep=0, minimum size=0, font=\scriptsize] {$\triangle$};

\draw[-, purple, thick] (vi1) -- (vi2);
\draw[-, purple, thick] (vi2) -- (v1);
\draw[-, purple, thick] (v1) -- (vi1);

\draw[-, orange, thick] (v1_2) -- (vi2_2);
\draw[-, orange, thick] (vi2_2) -- (vi3_2);
\draw[-, orange, thick] (v1_2) -- (vi3_2);

\draw[-, black, thick] (v1_3) -- (vi1_3);
\draw[-, black, thick] (v1_3) -- (vi3_3);
\draw[-, black, thick] (vi2_3) -- (vi3_3);
\draw[-, black, thick] (vi1_3) -- (vi2_3);

\draw[-, teal, thick] (v1_4) -- (vi1_4);
\draw[-, teal, thick] (vi1_4) -- (vi3_4);
\draw[-, teal, thick] (v1_4) -- (vi3_4);

\draw[->] (-5,-2.8) -- (-3.7,-2.8) node[midway, above, draw=none, shape=rectangle, inner sep=0, minimum size=0, font=\scriptsize] {$\triangle$};

\draw[-, black,very thick] (vi1_5) -- (vi2_5);
\draw[-, black,very thick] (vi1_5) -- (vi3_5);
\draw[-, black,very thick] (vi3_5) -- (vi2_5);

\end{tikzpicture}
\end{subfigure}
\caption{Illustration of how the triangle $T$ in the proof of Proposition~\ref{th:existence-balanced} can be obtained as $T=T_1'\triangle T_2' \triangle T_3'$.}
\label{fig:symDiff}
\end{figure}
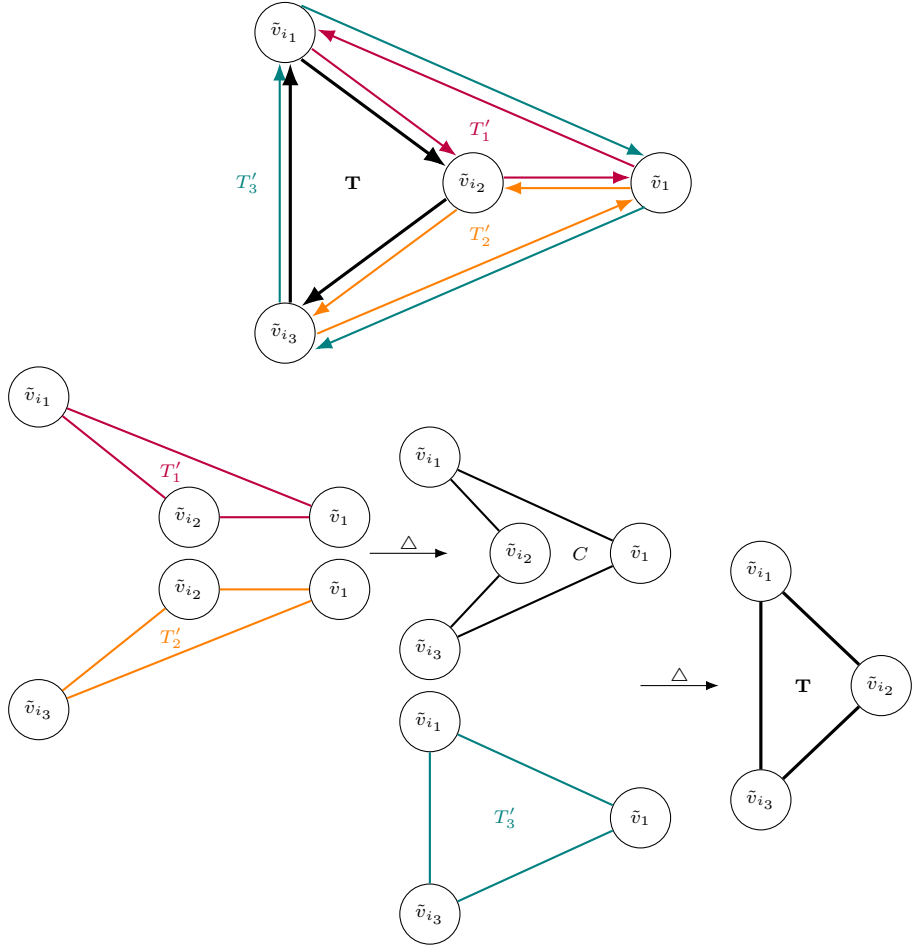

    Denote by $\theta_{*j}$ the phase of the weight $\tilde W_{1i_j}$ and $\theta_{j*}$ the phase of the weight $\tilde W_{i_j1}$ (then, $\theta_{j*}=-\theta_{*j}\mod 2\pi$). Then, since $T'_u$ are elements of the basis, we have:

    \[(\theta_{12}+\theta_{2*}+\theta_{*1})\operatorname{mod}2\pi=0\iff (\theta_{12}-\theta_{*2}+\theta_{*1})\operatorname{mod}2\pi=0\]
    \[(\theta_{23}+\theta_{3*}+\theta_{*2})\operatorname{mod}2\pi=0 \iff (\theta_{23}-\theta_{*3}+\theta_{*2})\operatorname{mod}2\pi=0\]
    \[(\theta_{31}+\theta_{1*}+\theta_{*3})\operatorname{mod}2\pi=0\iff (\theta_{31}-\theta_{*1}+\theta_{*3})\operatorname{mod}2\pi=0\]
    Then:
    \[0=(\theta_{12}-\theta_{*2}+\theta_{*1}+\theta_{23}-\theta_{*3}+\theta_{*2}+\theta_{31}-\theta_{*1}+\theta_{*3})\operatorname{mod}2\pi=(\theta_{12}+\theta_{23}+\theta_{31})\operatorname{mod}2\pi\]

    Note that a similar argument can be made if we fix the opposite orientation in $T$. Therefore, we have proven that every triangle in $\tilde G$ satisfies (iii).

    Finally, let us prove that every cycle in $\tilde G$ satisfies (iii).
    Consider an oriented cycle $(\tilde v_{i_1},\dots,\tilde v_{i_r})$ and denote by $\theta_{kl}$ the phase corresponding to the weight $\tilde W_{i_ki_l}$. Then: 
    \begin{align*}
       (\theta_{12}+\theta_{23}+\dots+\theta_{(r-1)r}+\theta_{r1})\operatorname{mod}2\pi&=(\theta_{13}+\theta_{34}+\dots+\theta_{(r-1)r}+\theta_{r1})\operatorname{mod}2\pi\\&=(\theta_{14}+\theta_{45}+\dots+\theta_{(r-1)r}+\theta_{r1})\operatorname{mod}2\pi\\&\dots\\&=(\theta_{1(r-1)}+\theta_{(r-1)r}+\theta_{r1})\operatorname{mod}2\pi=0, 
    \end{align*}
    where we have each equality by substituting the sum of the weights of a $2$-length path between two nodes by the weight of the edge joining the nodes (which forms a triangle). This process is illustrated in Figure \ref{fig:prop3cycle}. The last equality holds because $(\tilde v_{i_1i_{r-1}},\tilde v_{i_{r-1}i_r},\tilde v_{i_ri_1})$ is a triangle.

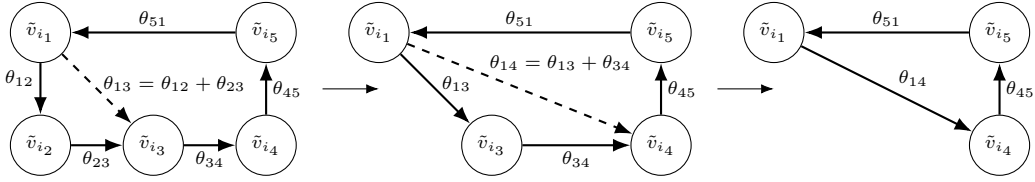
\begin{figure}
\centering
\begin{tikzpicture}[>=Latex, scale=0.75,
  every node/.style={
    circle,
    draw,
    minimum size=0.8cm,
    inner sep=1pt,
    font=\scriptsize
  }
]

\node (vi1) at (-17,2) {$\tilde{v}_{i_1}$};
\node (vi2) at (-17,0) {$\tilde{v}_{i_2}$};
\node (vi3) at (-15,0) {$\tilde{v}_{i_3}$};
\node (vi4) at (-13,0) {$\tilde{v}_{i_4}$};
\node (vi5) at (-13,2) {$\tilde{v}_{i_5}$};
\draw[->, black, thick] (vi1) -- (vi2) node[midway, left,yshift=4pt,xshift=-2pt, draw=none, shape=rectangle, inner sep=0, minimum size=0, font=\scriptsize] {$\theta_{12}$};
\draw[->, black, thick] (vi2) -- (vi3) node[midway, below,yshift=-2pt, draw=none, shape=rectangle, inner sep=0, minimum size=0, font=\scriptsize]{$\theta_{23}$};
\draw[->, black, thick] (vi3) -- (vi4)node[midway, below,yshift=-2pt, draw=none, shape=rectangle, inner sep=0, minimum size=0, font=\scriptsize]{$\theta_{34}$};
\draw[->, black, thick] (vi4) -- (vi5)node[midway, right,xshift=2pt, draw=none, shape=rectangle, inner sep=0, minimum size=0, font=\scriptsize] {$\theta_{45}$};
\draw[->, black, thick] (vi5) -- (vi1) node[midway, above,yshift=2pt, draw=none, shape=rectangle, inner sep=0, minimum size=0, font=\scriptsize]{$\theta_{51}$};
\draw[->, black, thick,dashed] (vi1) -- (vi3) node[midway, right,yshift=2pt,xshift=2pt, draw=none, shape=rectangle, inner sep=0, minimum size=0, font=\scriptsize] {$\theta_{13}=\theta_{12}+\theta_{23}$};

\draw[->] (-12,1) -- (-11,1);

\node (vi1_2) at (-11,2) {$\tilde{v}_{i_1}$};
\node (vi3_2) at (-9,0) {$\tilde{v}_{i_3}$};
\node (vi4_2) at (-6,0) {$\tilde{v}_{i_4}$};
\node (vi5_2) at (-6,2) {$\tilde{v}_{i_5}$};

\draw[->, black, thick] (vi3_2) -- (vi4_2)node[midway, below,yshift=-2pt, draw=none, shape=rectangle, inner sep=0, minimum size=0, font=\scriptsize]{$\theta_{34}$};
\draw[->, black, thick] (vi4_2) -- (vi5_2)node[midway, right,xshift=2pt, draw=none, shape=rectangle, inner sep=0, minimum size=0, font=\scriptsize] {$\theta_{45}$};
\draw[->, black, thick] (vi5_2) -- (vi1_2) node[midway, above,yshift=2pt, draw=none, shape=rectangle, inner sep=0, minimum size=0, font=\scriptsize]{$\theta_{51}$};
\draw[->, black, thick] (vi1_2) -- (vi3_2) node[midway, right,yshift=2pt,xshift=2pt, draw=none, shape=rectangle, inner sep=0, minimum size=0, font=\scriptsize] {$\theta_{13}$};
\draw[->, black, thick,dashed] (vi1_2) -- (vi4_2) node[midway, right,yshift=10pt,xshift=-12pt, draw=none, shape=rectangle, inner sep=0, minimum size=0, font=\scriptsize] {$\theta_{14}=\theta_{13}+\theta_{34}$};
\draw[->] (-5,1) -- (-4,1);

\node (vi1_3) at (-4,2) {$\tilde{v}_{i_1}$};
\node (vi4_3) at (0,0) {$\tilde{v}_{i_4}$};
\node (vi5_3) at (0,2) {$\tilde{v}_{i_5}$};

\draw[->, black, thick] (vi4_3) -- (vi5_3)node[midway, right,xshift=2pt, draw=none, shape=rectangle, inner sep=0, minimum size=0, font=\scriptsize] {$\theta_{45}$};
\draw[->, black, thick] (vi5_3) -- (vi1_3) node[midway, above,yshift=2pt, draw=none, shape=rectangle, inner sep=0, minimum size=0, font=\scriptsize]{$\theta_{51}$};
\draw[->, black, thick] (vi1_3) -- (vi4_3) node[midway, right,xshift=4pt,yshift=4pt, draw=none, shape=rectangle, inner sep=0, minimum size=0, font=\scriptsize] {$\theta_{14}$};

\end{tikzpicture}
\caption{Illustration of the proof of Proposition~\ref{th:existence-balanced}: every cycle satisfies property (iii) of Proposition~\ref{th:charact-balanced}.}
\label{fig:prop3cycle}
\end{figure}

    Therefore, we have proven that every cycle in $\tilde G$ satisfies (iii). 
    
    Now consider $G=(V,E)$, denote its nodes by $v_1,\dots,v_n$ and define the function $\sigma$, where $\sigma(j)$ returns the index $i$ of the node subset $V_i$ to which node $v_j$ belongs. Construct the matrix $\mb W=(W_{ij})$ by assigning weights in the following way:
    \[
    W_{ij}=\begin{cases}
        0, & \text{if }(v_i,v_j)\notin E\\
        1, & \text{if }\sigma(i)=\sigma(j)\\
        \tilde W_{\sigma(i)\sigma(j)} &\text{otherwise}
    \end{cases}
    \]

By construction, $G'=(V,E,\mb W)$ satisfies all the conditions of the Proposition. 
\end{proof}

Therefore, we have proven that every graph can be assigned complex weights so that the resulting complex-weighted graph satisfies the hypothesis of Proposition~\ref{th:convergente-rw}. Thus, it is possible to describe its asymptotic behaviour in the time limit of a complex random walk.

\MainCorollary*
\begin{proof}
    First, it is possible to assign complex weights to $G$ obtaining a balanced graph $G'=(V,E,\mb W)$ that satisfies the conditions of Proposition~\ref{th:existence-balanced}. In addition, we assign self-loops to every node, $W_{ii}=1\quad\forall i=1,\dots,n$, to ensure that $G'$ is not bipartite.
    
    Then, $G'$ is in the conditions of Proposition~\ref{th:convergente-rw}, so we can obtain its steady state in a complex random walk using Equation~\eqref{eq:steadystate}. Note that the factor $\frac{\mb x(0)^T\tilde{\mathbf{1}}_1}{2m}$ does not depend on $j$, so it is common for all nodes. Therefore, two nodes $v_j$ and $v_k$ have different phases iff $e^{i\theta_{1\sigma(j)}}\neq e^{i\theta_{1\sigma(k)}}$. 
    
    By construction of $G'$, $\theta_{1\sigma(j)}\mod 2\pi=\theta_{1\sigma(k)}\mod 2\pi$ iff $\sigma(j)=\sigma(k)$. Then, two nodes have the same phase iff they belong to the same subset of $\{V_i\}_{i=1}^{l_p}$.
\end{proof}

\section{Complex Weights Learning Proof}\label{app:cw-learning}
\PropLearnW*
\begin{proof}
    Let $\mb W\in\C^{n\times n}$ be a complex weight matrix for $G$, and denote its elements by $W_{ij}=a_{ij}+ib_{ij}$. Since $\mb W$ is Hermitian, $a_{ji}=a_{ij}$ and $b_{ji}=-b_{ij}$.

    Consider the feature matrix $\mb X$ as a real-valued matrix $\mb X\in\R^{n\times k}$, with $k=2f$.
    
    Define the set $S:=\{(\mb x_i,\mb x_j): (v_i,v_j)\in E\}\subset\R^{2k}$. Since each $(\mb x_i,\mb x_j)$ is unique, the function $g\colon S\to\R^{2}$, $g(\mb x_i,\mb x_j)=(a_{ij},b_{ij})$ is well-defined. We now show that $g$ can be extended to a smooth function $f\colon\R^{2k}\to\R^{2}$.
    
    Since $S=\{s_1,\dots,s_M\}$ is a finite set, for each $ s_m\in S$, there exists a sufficiently small neighbourhood $U_m\subset\R^{2k}$ such that $ s_m\in U_m$ and $s_p\notin U_p$ for all $m\neq p$. In addition, it is possible to find a smooth bump function $\psi_m\colon \R^{2k}\to\R$ such that $\psi_m( s_m)=1$ and $\psi_m( x)=0\quad\forall x\notin U_m$. Then, the function $f(x)=\sum_{m=1}^Mg( s_m)\psi_m( x)$ is smooth and $f|_S=g$. Since $f$ is smooth, it can approximated by an MLP \cite{hornik1991approximation,hornik1989multilayer} and, thus, so does $g$.

    Now it is enough to identify $\R^2\equiv \C$, with the first coordinate corresponding to the real part and the second coordinate to the imaginary part. Then, interpreting $S\subset\C^{2f}$, we have proven that $g\colon S\to\C$, $g(\mb x_i,\mb x_j)=a_{ij}+ib_{ij}$ can be approximated by an MLP with sufficient capacity. Therefore, $\Phi$ defined in Equation~\eqref{eq:def-phi} can learn any complex weight matrix $\mb W$ of the graph $G$.
\end{proof}
\section{Additional implementation details and hyperparameters}
\paragraph{Adjusting the activation magnitudes.} Following the approach in \cite{bodnar2022neural}, we found it useful in practice to learn additional parameters $\varepsilon^l=(\varepsilon^l_1, \varepsilon^l_2)$, with $\varepsilon^l_i\in[-1,1]$, for each layer $l$. These parameters are used to compute the diffusion step as follows:
\[
    \mb{X}(l+1)= (1+\varepsilon^l)\mb{X}(l)-\sigma( (\mb I-\mb P)(\mb I_n\otimes\new{\mb M}_1^l)\mb X(l)\new{\mb M}_2^l)\in\C^{n\times f}
\]
Here, $(1+\varepsilon^l)\mb{X}(l)$ denotes scaling the real part of $\mb X(l)$ by $\varepsilon^l_1+1$ and the imaginary part by $\varepsilon^l_2+1$. This mechanism allows the model to dynamically adjust the relative magnitudes of real and imaginary components at each layer. These learnable scaling parameters are used in all of our experiments.

\paragraph{Batch normalization.} Proposition~\ref{prop:universality} guarantees that any complex weight matrix can be learned, provided the node features are sufficiently diverse. However, in practice, we observed that the initial node features $\mb X(0)$ tend to be very similar across nodes. This issue arises from the initial MLP, where the bias term often dominates in practice, causing the output features of all nodes to become nearly identical.

To address this problem, we insert a batch normalization layer immediately after the MLP. This normalizes each feature dimension across the batch of nodes, mitigating bias dominance, and promoting feature diversity. Batch normalization is used in all of our experiments.

\paragraph{Real learnable matrices.} Complex matrix multiplication is more expensive than its real counterpart. To improve efficiency, we replace the complex matrices $\new{\mb M}_1^l\in\C, \new{\mb M}_2^l\in\C^{f\times f}$ in Equation~\eqref{eq:layerOp} with real matrices of size $2 \times 2$ and $2f \times 2f$, respectively. The feature matrix $\mb X(l)$ is then treated as a real matrix in $\R^{n \times 2f}$. In practice, we observe no loss in performance with this substitution, and therefore adopt it in all our experiments.

\paragraph{Hyperparameters and Training Procedure.} 
Following \cite{bodnar2022neural}, we evaluate our model using the hyperparameter ranges listed in Table~\ref{tab:hyperparam}, where dropout has been included as a regularization technique to prevent overfitting. We assign different dropout rates to the initial layer and to the linear layers within the convolutional blocks. We train all models for a fixed maximum number of epochs and perform early stopping when the validation metric has not improved for a pre-specified number of patience epochs. We report the test results at the epoch where the best validation metric was obtained for the model configuration with the best validation score among all models. We use the hyperparameter optimisation tools provided by Weights and Biases \cite{wandb} for this procedure.

\begin{table}
    \centering
    \caption{Hyperparameter ranges used to evaluate our model in the real-world experiments.}
    \label{tab:hyperparam}
    \begin{tabular}{l|c}
        \hline
        \textbf{Hyperparameter} & \textbf{Values}\\
        \hline
        Hidden channels & (8,16,32)\\
        Layers & 2-8\\
        Learning rate & 0.02 \\
        Activation (regular layers) & ELU \\
        Activation (complex weights learning) & tanh \\
        Weight decay (regular parameters)& Log-uniform $[0.007,0.2]$\\
        Weight decay (complex weights parameters)&  Log-uniform $[0.01,0.4]$ \\
        Input dropout & Uniform $[0,0.9]$\\
        Layer dropout &Uniform $[0,0.9]$\\ 
        Patience (epochs) &1000\\
        Max training epochs & 1500\\
        Optimiser & Adam\\
        \hline
    \end{tabular}
    
\end{table}

\section{\new{Additional Synthetic Experiments}}

\begin{figure}
        \centering
         \centering
    \begin{subfigure}[b]{0.67\textwidth}
        \centering
        \includegraphics[width=\textwidth]{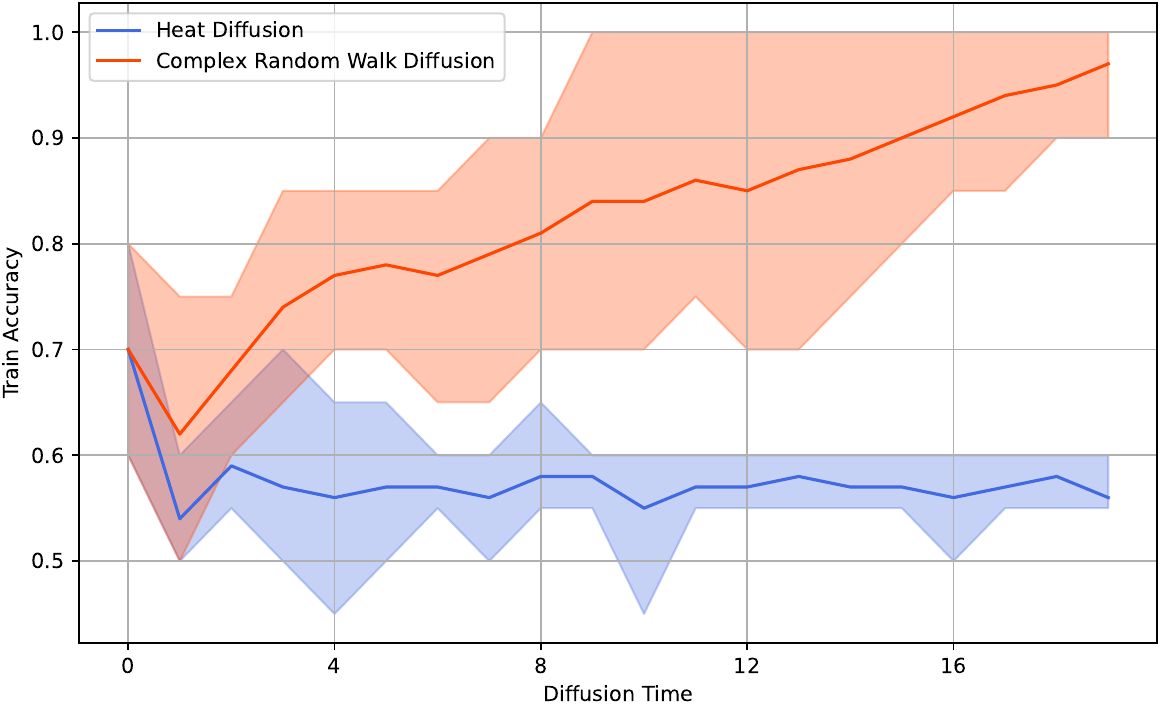}
    \end{subfigure}
    \hfill
    \begin{subfigure}[b]{0.27\textwidth}
        \centering
        \includegraphics[width=\textwidth]{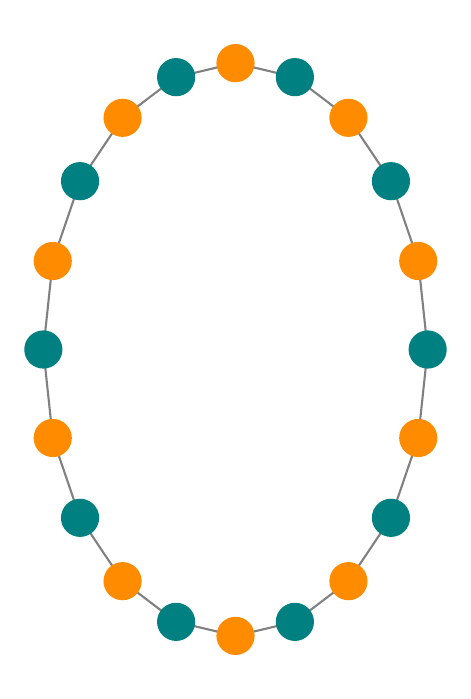}
    \end{subfigure}
        \caption{Training accuracy across diffusion layers (left) for heat diffusion and the learned complex random walk on a cycle bipartite graph with 2 classes (right): While the mean accuracy (solid line) for heat diffusion remains low throughout, the complex random walk surpasses 95 \% as diffusion progresses. The shaded region (minimum–maximum range across 5 random seeds) further shows that heat diffusion never exceeds 70 \%, whereas the complex random walk maintains over 90 \% mean accuracy across all seeds.}

         \label{fig:syntheticAcc}
    \end{figure}

In this section, we conduct further synthetic experiments that allow a direct comparison with standard heat diffusion in a controlled setting.

We consider a bipartite cycle graph with 20 nodes divided into two partitions (Figure~\ref{fig:syntheticAcc}, right), and assign node features sampled from two overlapping isotropic Gaussian distributions. This setup ensures that the classes are \new{not linearly separable} at initialization. As shown by \cite{li2018deeper}, heat diffusion fails to separate the classes in the diffusion limit. In contrast, Theorem~\ref{cor:linearSep} demonstrates that a complex random walk can achieve linear separation in its steady state. In this experiment, we study whether a suitable complex weight matrix can be learned directly from data in this simplified setting, using a vanilla random walk diffusion process, i.e., by setting $\new{\mb M}_1=1$, $\new{\mb M}_2=\mb I_f$ and $\sigma=\operatorname{id}$ in Equation~\eqref{eq:layerOp}. 

For heat diffusion, we first learn a transformation of the initial node features $\mb X(0)$, and then apply the diffusion process. For the complex random walk, we instead learn a complex-valued weight matrix $\mb W$ as a function of $\mb X(0)$, and subsequently perform the complex random walk. In both approaches, the model parameters are optimized to produce linearly separable features at the diffusion time limit.

Figure~\ref{fig:syntheticAcc} (left) shows the classification results averaged over five random seeds. As expected, at diffusion time zero, a linear classifier fails to distinguish the classes. As diffusion progresses, heat diffusion continues to yield non-separable features, whereas the complex random walk consistently achieves over 90\% mean accuracy across seeds. This result highlights the potential of learning an effective complex-weighted graph structure that enables successful node classification at the diffusion limit.

\end{document}